\documentclass[%
reprint,
nofootinbib,
 amsmath,amssymb,
 aps,
 prl,
]{revtex4-2}

\usepackage{graphicx}
\usepackage{dcolumn}
\usepackage{bm}
\usepackage{float}
\usepackage{subcaption}
\usepackage{multirow}
\usepackage{comment}
\usepackage{dsfont}
\usepackage{amsthm}
\usepackage{xcolor}
\usepackage{array}
\usepackage{eucal}
\usepackage{makecell}
\usepackage{mathtools}
\usepackage{makecell}

\def\ie{{\frenchspacing\it i.e.}}
\def\eg{{\frenchspacing\it e.g.}}

\newcommand{\mat}[1]{\mathbf{#1}}

\def\diff{{\rm diff}}

\def\f{\textbf{f}}
\def\g{{\bf g}}
\def\h{{\bf h}}

\def\I{{\bf I}}
\def\A{{\bf A}}
\def\G{{\bf G}}
\def\J{{\bf J}}
\def\M{{\bf M}}
\def\SS{{\bf S}}
\def\T{{\bf T}}
\def\U{{\bf U}}

\def\r{\textbf{r}}

\def\x{\textbf{x}}

\def\z{{\textbf{z}}}
\def\zhat{\hat\z}
\def\zdot{\dot{\z}}
\def\fNN{\f_{\rm NN}}
\def\m{\textbf{m}}
\def\p{\textbf{p}}

\def\W{\mat{W}}
\def\T{\mat{T}}
\def\hl{\hat{L}}
\def\jth{\hat{j^{\rm th}}}

\newcommand{\zm}[1]{\textcolor{blue}{#1}}

\newtheorem{lemma}{Lemma}

\def\spose#1{\hbox to 0pt{#1\hss}}
\def\simlt{\mathrel{\spose{\lower 3pt\hbox{$\mathchar"218$}}
     \raise 2.0pt\hbox{$\mathchar"13C$}}}
\def\simgt{\mathrel{\spose{\lower 3pt\hbox{$\mathchar"218$}}
     \raise 2.0pt\hbox{$\mathchar"13E$}}}
\def\simpropto{\mathrel{\spose{\lower 3pt\hbox{$\mathchar"218$}}
     \raise 2.0pt\hbox{$\propto$}}}

\def\beq#1{\begin{equation}\label{#1}}
\def\eeq{\end{equation}}
\def\beqa#1{\begin{eqnarray}\label{#1}}
\def\eeqa{\end{eqnarray}}

\def\eqnum#1{~(\ref{#1})}

\def\fig#1{Fig.~\ref{#1}}
\def\Fig#1{Figure~\ref{#1}}

\def\tabl#1{Tab.~\ref{#1}}
\def\Tabl#1{Table~\ref{#1}}

\newtheorem{theorem}{Theorem}

\begin{document}

\title{Machine-learning hidden symmetries}

\author{Ziming Liu and Max Tegmark}
\affiliation{%
Department of Physics, Massachusetts Institute of Technology, Cambridge, USA
}%

\date{\today}

\begin{abstract}
We present an automated method for finding hidden symmetries, defined as symmetries that become manifest only in a new coordinate
system that must be discovered. Its core idea is to  quantify asymmetry as violation of certain partial differential equations, and to numerically minimize such violation over the space of all invertible transformations, parametrized as invertible neural networks. 
For example, our method rediscovers the famous Gullstrand-Painlev\'e metric that
manifests hidden translational symmetry in the Schwarzschild metric of non-rotating black holes, as well as  
Hamiltonicity, modularity and other simplifying traits not traditionally viewed as symmetries.
\end{abstract}

\maketitle

\section{Introduction}

Philip Anderson famously argued that ``It is only slightly overstating the case to say that physics is the study of symmetry" \cite{PhilipAnderson1972}, and discovering symmetries has proven enormously useful both for deepening understanding and for solving problems more efficiently, in 
physics \cite{Wigner,PhilipAnderson1972,Wilczek} as well as machine learning \cite{cohen2016group, thomas2018tensor,fuchs2020se,kondor2018clebsch,satorras2021n,phialaprl,phialasun}. 

Discovering symmetries is useful but hard, because they are often not {\it manifest} but  {\it hidden}, becoming manifest only after an appropriate coordinate transformation. 
For example, after Schwarzschild discovered his eponymous black hole metric, 
it took 17 years until 
Painlev\'e, 
Gullstrand 
and Lema\^itre 
showed that it had hidden translational symmetry: 
they found that the spatial sections could be made translationally invariant with a clever coordinate transformation, thereby deepening our understanding of black holes \cite{MisnerThorneWheelerBook}. As a simpler example, \fig{fig:1d_ho_plot} shows the same vector field in two coordinates systems where rotational symmetry is manifest and hidden, respectively. 

Our results below are broadly applicable because they apply to a very broad definition of symmetry, including not only {\it invariance} and {\it equivariance} with respect to arbitrary Lie groups, but also {\it modularity} and {\it Hamiltonicity}. If a coordinate transformation is discovered that makes such simplifying properties manifest, this can not only deepen our understanding of the system in question, but also enable an arsenal of more efficient numerical methods for studying it.

Discovering hidden symmetries is unfortunately highly non-trivial, because it involves a search over  
all smooth invertible coordinate transformations, and has traditionally been accomplished by scientists making inspired guesses. The goal of this {\it Letter} is to present a machine learning method for automating hidden symmetry discovery. Its core idea is to  quantify asymmetry as violation of certain partial differential equations, and to numerically minimize such violation over the space of all invertible transformations, parametrized as invertible neural networks. For example, the neural network automatically learns to transform \fig{fig:1d_ho_plot}(b) into \fig{fig:1d_ho_plot}(c), thereby making the hidden rotational symmetry manifest. Our method differs from previous work~\cite{PINN,prl_enforcing,LNN,NNPhD,phialaprl,phialasun} that exploits manifest symmetries, partial differential equations or other physical properties to facilitate machine learning, but not the other way around to discover hidden symmetries with machine learning as a tool.

In the Method section, we introduce our notation and symmetry definition and present our method for hidden symmetry discovery. In the Results section, we apply our method  to classical mechanics and general relativity examples to test its ability to auto-discover hidden symmetries.

\begin{figure}
    \centering
    \includegraphics[width=1\linewidth]{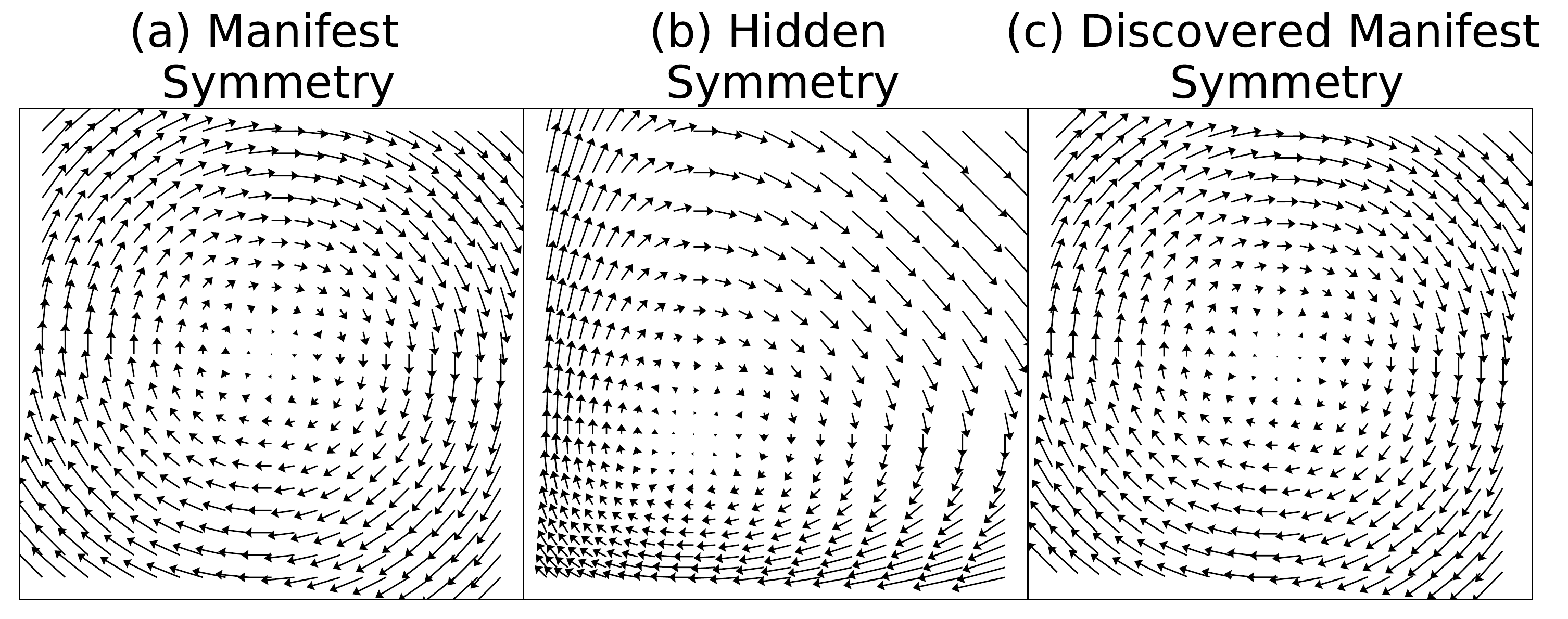}
    \caption{1D harmonic oscillator phase space flow vector field  $\f(x,p)=(p,-x)$. The rotational symmetry of $\f$ is manifest in (a) and hidden in (b). Our algorithm can reveal the hidden symmetry by auto-discovering the transformation from (b) to (c).}
    \label{fig:1d_ho_plot}
\end{figure}

\begin{figure}
    \centering
    \includegraphics[width=0.75\linewidth]{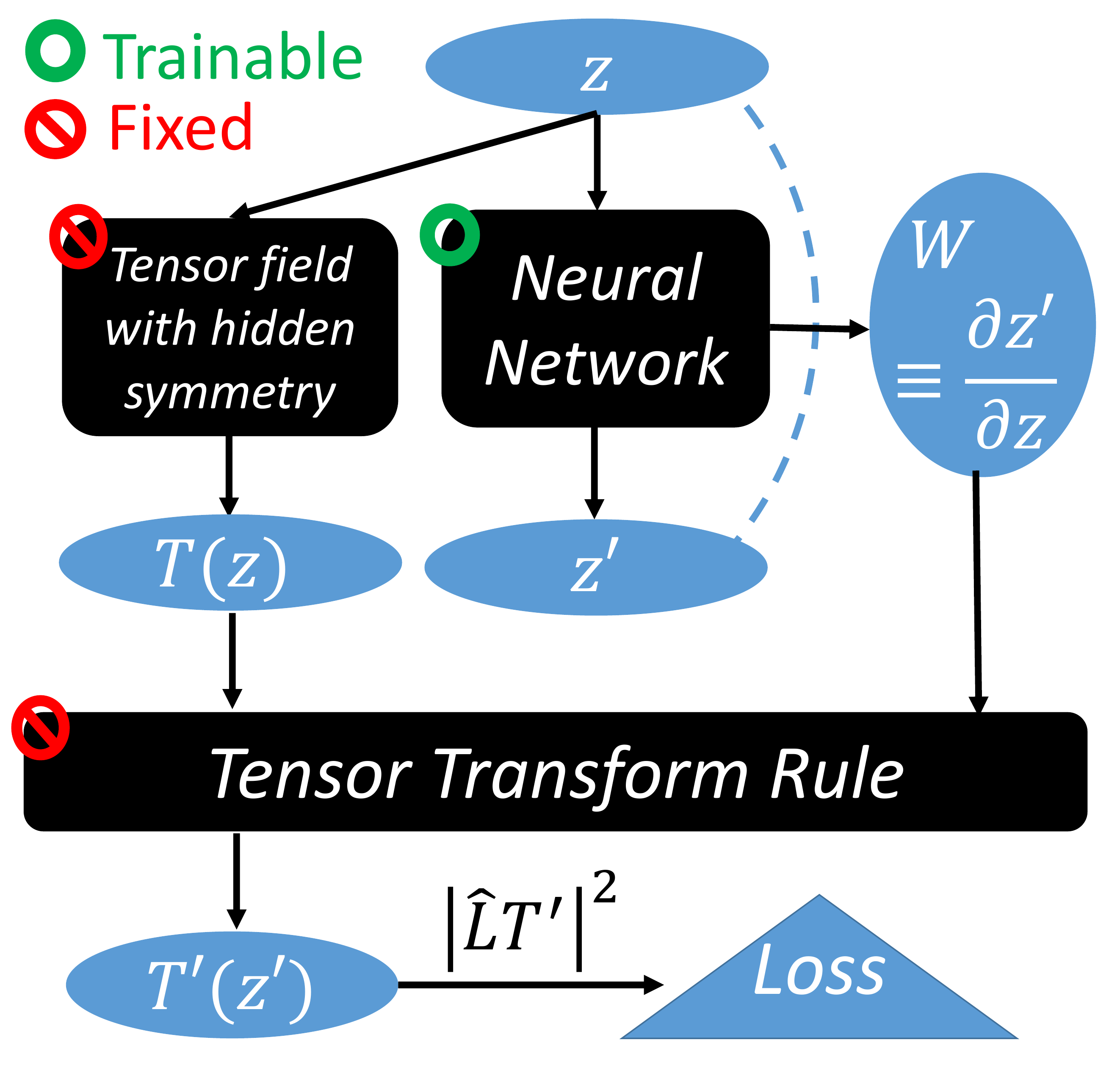}
    \caption{Schematic workflow of our algorithm for discovering hidden symmetry}
    \label{fig:workflow}
\end{figure}

\section{Method}

\subsection{PDEs encoding generalized symmetries}

\begin{table}[]
    \centering
    \caption{PDE and Losses for Generalized Symmetries}
    \resizebox{0.5\textwidth}{!}{%
    \begin{tabular}{|c|c|c|c|}\hline
    Generalized symmetry & Linear operator $\hl$ & Loss $\ell$ & Examples \\\hline
    Translation invariance& $\hl_j=\partial_j$ & $\ell_{\rm TI}$ & A,E,F \\\hline
    Lie invariance & $\hl_j=K_j\z\cdot\nabla$ & $\ell_{\rm INV}$ & E,F \\\hline
    Lie equivariance & $\hl_j=K_j\z\cdot\nabla\pm K_j$ & $\ell_{\rm EQV}$ & B \\\hline 
    Canonical eqvariance & \makecell{$\hl_{j}^{\x}=K_j\x\cdot\nabla_{\x}-K_j^t\p\cdot\nabla_{\p}+K_j^t$ \\ $\hl_{j}^{\p}=K_j\x\cdot\nabla_{\x}-K_j^t\p\cdot\nabla_{\p}-K_j$} & $\ell_{\rm CAN}$ & C \\\hline 
    Hamiltonicity  & $\hl_{ij}=-\mat{m}_i^t\partial_j+\mat{m}_j^t\partial_i$ & $\ell_{\rm H}$ & A,B,C,D\\\hline
    Modularity & $\hl_{ij}=\mat{A}_{ij}\hat{\z}_i^t\partial_j$ & $\ell_{\rm M}$ & D \\\hline
    \end{tabular}}
    \label{tab:pde}
\end{table}

We seek to discover symmetries in various tensor fields $T(\z)$ for $\z\in\mathbb{R}^n$, for example the vector 
field $\f(\z)$ (a rank-1 tensor) defining a dynamical system $\z(t)$ through a vector differential equation $\dot\z=\f(\z)$, or the metric $g(\z)$ (a rank-2 tensor) quantifying spacetime geometry in general relativity.
We say that a tensor field $T$ has
a {\it generalized symmetry} if it obeys a linear partial differential equation (PDE) $\hl\,T=0$,
where $\hl$ is a linear operator that encodes the symmetry generators.
This definition covers a broad range of interesting situations, as illustrated by the examples below (see Table \ref{tab:pde} for a summary).

{\bf Translational Invariance}: 
A tensor field $T$ is invariant under translation in the $\jth$ coordinate direction $\zhat_j$ if 
$T(\z+a\zhat_i) = T(\z)$ for all $\z$ and $a$, which is equivalent to satisfying the PDE 
$\partial T/\partial z_j=0$, corresponding to the linear operator $\hl=\partial_j$.

{\bf Lie invariance \& equivariance}: 
If $T(\z)$ satisfies $T(g\z)=g^nT(\z)$ for all elements 
$g$ of some Lie group $\mathcal{G}$ and an integer $n=-1$, 0 or 1, 
then we say that $T$ is {\it invariant} if $n=0$, and {\it equivariant} otherwise ($n=1$ corresponds to a {\it covariant} (1,0) vector field and $n=-1$ corresponds to a {\it contravariant} (0,1) vector field)~\footnote{An $(m,n)$-tensor means the tensor has $m$ covariant and $n$ contravariant indices, which is a convenient notations especially in general relativity when dealing with metrics. We call a (1,0)-tensor a covariant vector, and a (0,1)-tensor a contravariant vector.}.
Taking the derivative on the both sides of the identity $T(e^{K_j a}\z)=e^{n K_j a} T(\z)$ with respect to $a$ at $a=0$ gives the PDEs $\hl_j\,\f=0$ with 
$\hl_j\equiv K_j\z\cdot \nabla - n K_j$. \Fig{fig:1d_ho_plot} (a) and (c) show examples of rotational equivariance.

{\bf Hamiltonicity} (a.k.a. {\it symplecticity}): A dynamical system $\z(t)\in\mathbb{R}^{2d}$ obeying a vector differential equation 
$\zdot=\f(\z)$ is called {\it Hamiltonian} or {\it symplectic} if $\f=\M\nabla H$ for a scalar function $H$, where
\beq{MdefEq}
\M\equiv\left(
\begin{tabular}{cc}
0&\I\\
-\I&0
\end{tabular}\right),
\eeq
and $\I$ is the $d\times d$ identity matrix. Such systems are of great importance in physics, where it is customary to write $\z=(\x,\p)$, because
the Hamiltonian function $H(\z)$ (interpreted as energy) is a conserved quantity under the system evolution 
$\zdot=\f(\z)=(\dot{\x}, \dot{\p})=\M\nabla H=(\partial_\p H, -\partial_\x H)$.
Hamiltonicity thus corresponds to $\M^{-1}\f$ being a gradient, \ie, to its generalized curl (the antisymmetric parts of its Jacobian matrix) vanishing.
Letting $\J\equiv\nabla\f$ denote the Jacobian of $\f$  ($J_{ij}\equiv f_{i,j}$) and using the fact that $\M^{-1}=\M^t=-\M$ (superscript $t$ denotes transpose), Hamiltonicity is thus equivalent to satisfying the PDEs
$\hl_{ij} \f=0$ where $\hl_{ij} \f=(\M\J+\J^t\M)_{ij}$ for all $i$ and $j$ ($n(n-1)\over 2$ independent  ODE's in all),
corresponding to 
$\hl_{ij}=-\m_i^t\partial_j+\m_j^t\partial_i$, 
where $\m_i$ are the column vectors of $\M$.
In other words, although Hamiltonicity is not traditionally thought of as a symmetry, it conveniently meets our generalized symmetry definition and can thus be auto-discovered with our method.

{\bf Canonical equivariance}: We define a Hamiltonian system as canonically equivariant if $\z=(\x,\p)$ and the vector field $\f\equiv(\f_{\x},\f_{\p})$ satisfies $\f_{\x}(g\x,g^{-1}\p)=g^{-t}\f_{\x}$ and $\f_{\p}(g\x,g^{-1}\p)=g\f_{\p}$ for all $g\in\mathcal{G}$.
These two equations are equivalent to the PDEs $\hl_j^{\x}\f_{\x}=0$ and $\hl_j^{\p}\f_{\p}=0$ with $\hl^{\x}_j=K_j\x\cdot\nabla_{\x}-K_j^t\p\cdot\nabla_{\p}+K_j^t$ and $\hl^{\p}_j=K_j\x\cdot\nabla_{\x}-K_j^t\p\cdot\nabla_{\p}-K_j$. In special cases when the generator $K_j$ is anti-symmetric (\eg, for the rotation group), $\hl_j^{\x}=\hl_j^{\p}$.

{\bf Modularity}: A dynamical system $\z(t)$ obeying $\zdot=f(\z)$ is {\it modular} if the Jacobian $\J=\nabla\f$ is block-diagonal, which implies that the components of $\z$ corresponding to different blocks evolve independently of each other.
More generally, we say that a system is $(n_1+\cdots+n_k)$-modular if $\J$ vanishes except for blocks of size $n_1,...n_k$, which we can write as 
$\A\circ\J=0$ where $\circ$ denotes element-wise multiplication ($([\A\circ \J])_{ij}=\A_{ij}\J_{ij}$) and the elements of the mask matrix $\A$ equal 1 inside the blocks, vanishing otherwise.
Although modularity is typically not viewed as a symmetry, it too thus meets our generalized symmetry definition and can be auto-discovered with our method using the matrix PDE $\A\circ\nabla\f=0$, corresponding to the linear operators
$\hl_{ij}\equiv \A_{ij}\zhat_i^t\partial_j$.

\subsection{Our algorithm for discovering hidden symmetries}

We now describe our algorithm of discovering hidden symmetries. Since $\hl T=0$ implies manifest symmetry, $|\hl T|^2$ is a natural measure of manifest symmetry violation. We therefore define the {\it symmetry loss} as 
\beq{lDefEq}
\ell\equiv {\langle |\hl T(\z)|^2\rangle\over\langle |\z|^2\rangle^\alpha},
\eeq
where angle brackets denote averaging over some set of points $\z_i$, and $\alpha$ is chosen so that $\ell$ is scale-invariant \ie, 
invariant under a scale transformation $\z\to a\z$, $T\to a^{m-n}T$, $\hat{L}\to a^s \hat{L}$, $\ell\to a^{2(m-n+s-\alpha)}$ if $T$ has $m$ contravariate indices and $n$ covariate indices. Hence we choose $\alpha=m-n+s$. 
We jointly search for multiple hidden symmetries by using the loss function $\ell=\sum_i\ell_i$ where each $i$ corresponds to one symmetry, denoted by subscripts as in \tabl{tab:pde}. 

Discovering hidden symmetry is equivalent to minimizing $\ell$ over all diffeomorphisms (everywhere differentiable and invertible coordinate transformations), which we parametrize with an invertible neural network. \Fig{fig:workflow} shows the workflow of our algorithm: (1) a neural network transforms $\z\mapsto\z'$ and obtains the transformation's Jacobian $\W\equiv d\z'/d\z$; (2) in parallel with (1), we evaluate the known tensor field $T$ at $\z$; (3) we  jointly feed $\W(\z)$ and $T(\z)$ into a module which implements the tensor transformation rule and  gives $T'(\z')$; (4) we compute the symmetry loss of $\ell(T')$. 
Note that only the neural network is trainable, while both the tensor field with hidden symmetry $T(\z)$ and tensor 
transformation rule are hard-coded in the workflow. We update the neural network with back-propagation to find the coordinate transformation 
$\z\mapsto\z'$ that minimizes $\ell$. If the resulting $\ell$ is effectively zero, a hidden symmetry has been discovered.

\subsection{Neural network training and symbolic regression} 

We parametrize the coordinate transformation $\z\mapsto\z'$ as 
$\z'=\z+\fNN(\z)$, where $\fNN$ is a fully connected neural network with two hidden layers containing 400 neurons each. We use a silu activation function~\cite{silu} rather than the popular ReLU alternative, because our method requires activation functions to be 
twice differentiable (since the loss function involves first derivatives of output with respect to input via the Jacobian $\W$). Derivatives of PDE losses and Jacobians are calculated with automatic differentiation and backpropagation. The invertibility of the mapping $\z\to\z'$ is guaranteed by the fact that $\det\W\to 0$ and the loss function $\ell\to\infty$ if $\z\to\z'$ approaches non-invertibility, as seen in equations\eqnum{TransformationEq1}-\eqnum{TransformationEq4} in the supplementary material.
The supplementary material also provides further technical details on the selection of data points $\z$, neural network initialization and training. If multiple symmetries are tested, the training process is performed in multiple stages: at each stage, we add one more symmetry to the loss function and re-train to convergence.

We then apply AI Feynman, a physics-inspired symbolic regression module, to interpret what the neural network has learned;
for details, see Appendix \ref{app:feynman} and~\cite{aifeynman,aifeynman2}.

\section{Results}

\begin{table*}[]
    \centering
    \caption{Physical Systems studied}
    \resizebox{\textwidth}{!}{%
    \begin{tabular}{|c|c|c|c|c|c|c|c|}\hline
    ID & Name  & Original dynamics $\dot\z=\f(\z)$ or metric $g(\z)$ & Transformation $\z\mapsto\z'$& Symmetric dynamics 
    $\dot\z'=\f'(\z')$ or metric $g'(\z')$ &  Manifest Symmetries \\\hline
    A & \makecell{1D Uniform \\ Motion }  & $\frac{d}{dt}\begin{pmatrix}
	a \\ b
	\end{pmatrix}=\begin{pmatrix}
	\frac{1}{2}(a+b){\rm ln}(\frac{a-b}{2})\\
	\frac{1}{2}(a+b){\rm ln}(\frac{a-b}{2})
	\end{pmatrix}$ & $\begin{pmatrix}
	a \\ b
	\end{pmatrix}=\begin{pmatrix}
	e^{\frac{1}{2}x} + e^{\frac{1}{2}p}\\ e^{\frac{1}{2}x} - e^{\frac{1}{2}p}
	\end{pmatrix}$ & $\frac{d}{dt}\begin{pmatrix}
    x\\ p
    \end{pmatrix}=\begin{pmatrix}
    p\\0
    \end{pmatrix}$ & \makecell{Hamiltonicity \\ 1D Translational invariance} \\\hline
    B & \makecell{1D Harmonic \\Oscillator} & $\frac{d}{dt}\begin{pmatrix}
    a \\ b 
    \end{pmatrix}=\begin{pmatrix}
    (1+a){\rm ln}(1+b) \\ 
    -(1+b){\rm ln}(1+a)
    \end{pmatrix}$ & $\frac{d}{dt}\begin{pmatrix}
    a \\ b
    \end{pmatrix}=\begin{pmatrix}
    e^{\frac{1}{2}x}-1 \\ e^{\frac{1}{2}p}-1
    \end{pmatrix}$ &$\frac{d}{dt}\begin{pmatrix}
    x \\ p
    \end{pmatrix}=\begin{pmatrix}
    p \\ -x
    \end{pmatrix}$ & \makecell{Hamiltonicity \\ SO(2)-equivariance} \\\hline
    C & 2D Kepler &  $\frac{d}{dt}\begin{pmatrix}
    a \\ b \\ c \\ d
    \end{pmatrix}=\begin{pmatrix}
    (1+a){\rm ln}(1+b) \\
    -\frac{(1+b){\rm ln}(1+a)}{8({\rm ln}^2(1+a)+{\rm ln}^2(1+c))^{3/2}}\\
    (1+c){\rm ln}(1+d) \\
    -\frac{(1+d){\rm ln}(1+c)}{8({\rm ln}^2(1+a)+{\rm ln}^2(1+c))^{3/2}}\\
    \end{pmatrix}$ &  $\begin{pmatrix}
    a \\ b \\ c \\ d
    \end{pmatrix}=\begin{pmatrix}
    e^{\frac{1}{2}x}-1 \\ e^{\frac{1}{2}p_x}-1 \\ e^{\frac{1}{2}y}-1 \\ e^{\frac{1}{2}p_y}-1
    \end{pmatrix}$  & $\frac{d}{dt}\begin{pmatrix}
    x \\ p_x \\ y \\ p_y
    \end{pmatrix}=\begin{pmatrix}
    p_x \\ -x/r^3 \\ p_y \\ -y/r^3
    \end{pmatrix}$ & \makecell{Hamiltonicity \\Can SO(2)-equivariance} \\\hline   
    D & \makecell{Double \\ Pendulum} & $\frac{d}{dt}\begin{pmatrix}
    \theta_1 \\ \dot{\theta}_1 \\ \theta_2 \\ \dot{\theta}_2
    \end{pmatrix}=\begin{pmatrix}
    \dot{\theta}_1 \\ -\frac{(m_1+m_2)g}{m_1l}\theta_1+\frac{m_2g}{m_1l}\theta_2 \\ \dot{\theta}_2 \\ \frac{(m_1+m_2)g}{m_1l}\theta_1-\frac{(m_1+m_2)g}{m_1l}\theta_2 
    \end{pmatrix}$ & \makecell{
    $\begin{pmatrix}
    \theta_1 \\ 
    \theta_2 \\ 
    \dot{\theta}_1 \\ \dot{\theta}_2
    \end{pmatrix}=\begin{pmatrix}
    -1 & 1 & 0 & 0 \\
     a & a & 0 & 0 \\
    0 & 0 & -1 & 1 \\
    0 & 0 & a & a
    \end{pmatrix}\begin{pmatrix}
    \theta_+ \\ \theta_- \\ \dot{\theta}_+ \\ \dot{\theta}_-
    \end{pmatrix}$ \\
    $a=\sqrt{\frac{m_1+m_2}{m_2}}$} &
    \makecell{$\frac{d}{dt}\begin{pmatrix}
    \theta_+ \\ \dot{\theta}_+ \\ \theta_- \\ \dot{\theta}_-
    \end{pmatrix}=\begin{pmatrix}
    \dot{\theta}_+ \\ -\omega_+^2\theta_+ \\ \dot{\theta}_- \\ -\omega_-^2\theta_-
    \end{pmatrix}$ \\  $\omega_{\pm}^2=\frac{(m_1+m_2)g}{m_1l}(1\pm\sqrt{\frac{m_2}{m_1+m_2}})$}  & \makecell{Hamiltonicity \\ $(2+2)$-Modularity} \\\hline
    E & \makecell{Expanding\\universe\\\& empty space} & $\g=\begin{pmatrix}
   	1 & 0 & 0 & 0 \\
   	0 & -(r^2+\frac{kx^2}{1-kr^2})\frac{t^2}{r^2} & -\frac{kxyt^2}{1-kr^2} & -\frac{kxzt^2}{1-kr^2} \\
   	0 & -\frac{kxyt^2}{1-kr^2} & -(r^2+\frac{ky^2}{1-kr^2})\frac{t^2}{r^2} & -\frac{kyzt^2}{1-kr^2} \\
   	0 & -\frac{kxzt^2}{1-kr^2} & -\frac{kyzt^2}{1-kr^2} & -(r^2+\frac{kz^2}{1-kr^2})\frac{t^2}{r^2}
   	\end{pmatrix}$ & $\begin{pmatrix} t' \\ x' \\ y' \\ z'
   	   \end{pmatrix}=\begin{pmatrix}
   	   t\sqrt{1+r^2} \\ tx \\ ty \\ tz
   	   \end{pmatrix}$ & $\g=\begin{pmatrix}
   	1 & 0 & 0 & 0 \\
   	0 & -1 & 0 & 0 \\
   	0 & 0 & -1 & 0 \\
   	0 & 0 & 0 & -1
   	\end{pmatrix}$ & \makecell{SO(3,1)-Invariance \\ 4D Translational Invariance}\\\hline
    F & \makecell{Schwarzchild\\black hole\\\& GP metric} & $\g=\begin{pmatrix}
    1-\frac{2M}{r} & 0 & 0 & 0 \\
    0 & -1-\frac{2Mx^2}{(r-2M)r^2} & -\frac{2Mxy}{(r-2M)r^2} & -\frac{2Mxz}{(r-2M)r^2} \\
    0 & -\frac{2Mxy}{(r-2M)r^2} & -1-\frac{2My^2}{(r-2M)r^2} & -\frac{2Myz}{(r-2M)r^2} \\
    0 & -\frac{2Mxz}{(r-2M)r^2} & -\frac{2Myz}{(r-2M)r^2} & -1-\frac{2Mz^2}{(r-2M)r^2}
    \end{pmatrix}$ & \makecell{
    	$\begin{pmatrix}
    		t' \\ x' \\ y' \\ z'
    	\end{pmatrix} = \begin{pmatrix}
    	t+2M\left[2u+\ln{u-1\over u+1}\right]\\
    	x\\ 
    	y\\ 
    	z
    	\end{pmatrix}$ \\ $u\equiv\sqrt{\frac{r}{2M}}$} &   $\g=\begin{pmatrix}
    	1-\frac{2M}{r} & -\sqrt{\frac{2M}{r}}\frac{x}{r} & -\sqrt{\frac{2M}{r}}\frac{y}{r} & -\sqrt{\frac{2M}{r}}\frac{z}{r} \\
    	-\sqrt{\frac{2M}{r}}\frac{x}{r} & -1 & 0 & 0 \\
    	-\sqrt{\frac{2M}{r}}\frac{y}{r} & 0 & -1 & 0 \\
    	-\sqrt{\frac{2M}{r}}\frac{z}{r} & 0 & 0 & -1
    	\end{pmatrix}$ & \makecell{SO(3)-Invariance \\ 3D Translational Invariance} \\\hline
    \end{tabular}}\\
    \label{tab:examples}
\end{table*}

\begin{figure*}[]
    \centering
    \resizebox{\textwidth}{!}{%
    \begin{tabular}{|c|c|c|}\hline
	\includegraphics[width=0.4\linewidth]{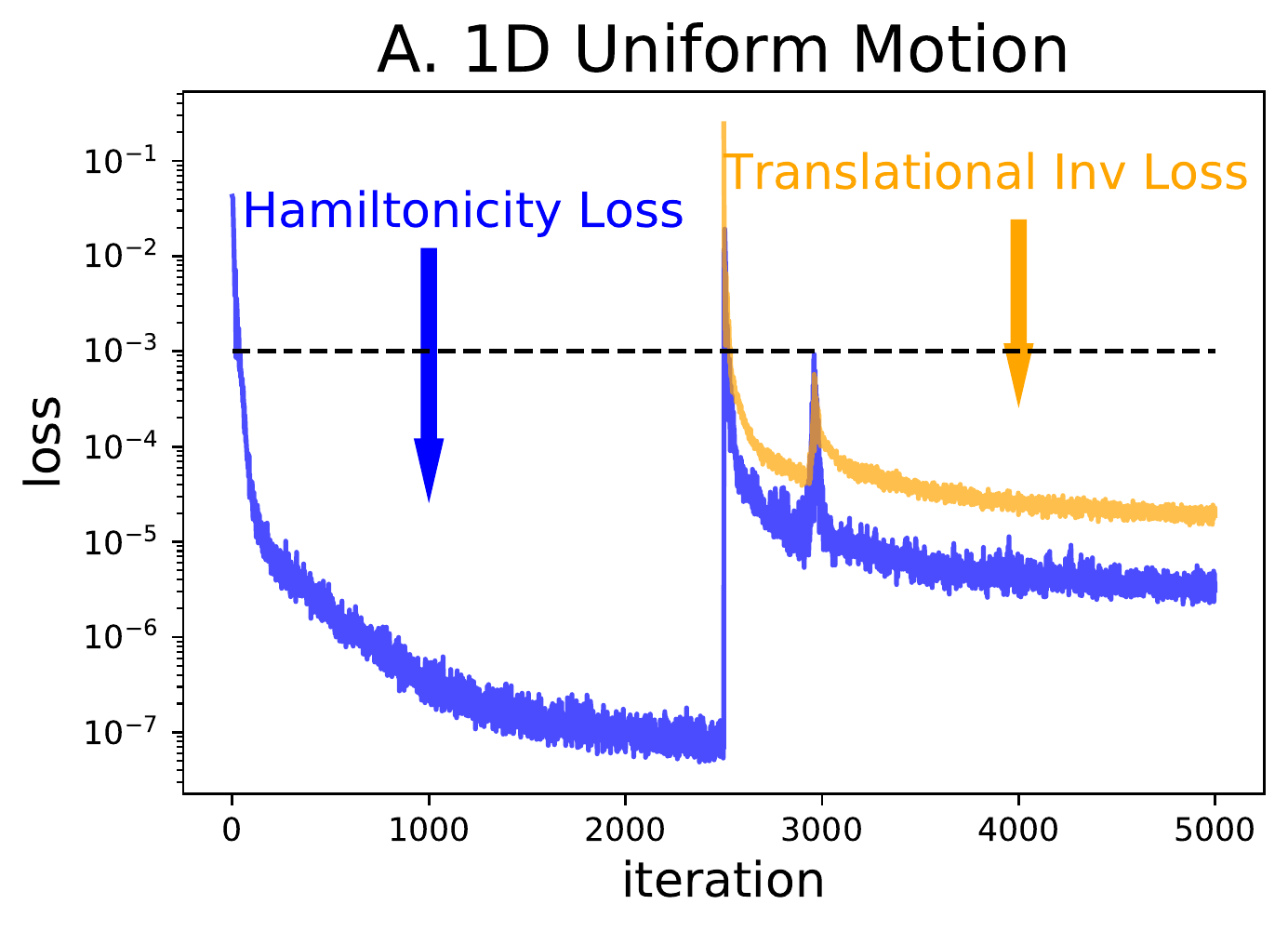}
	 & \includegraphics[width=0.4\linewidth]{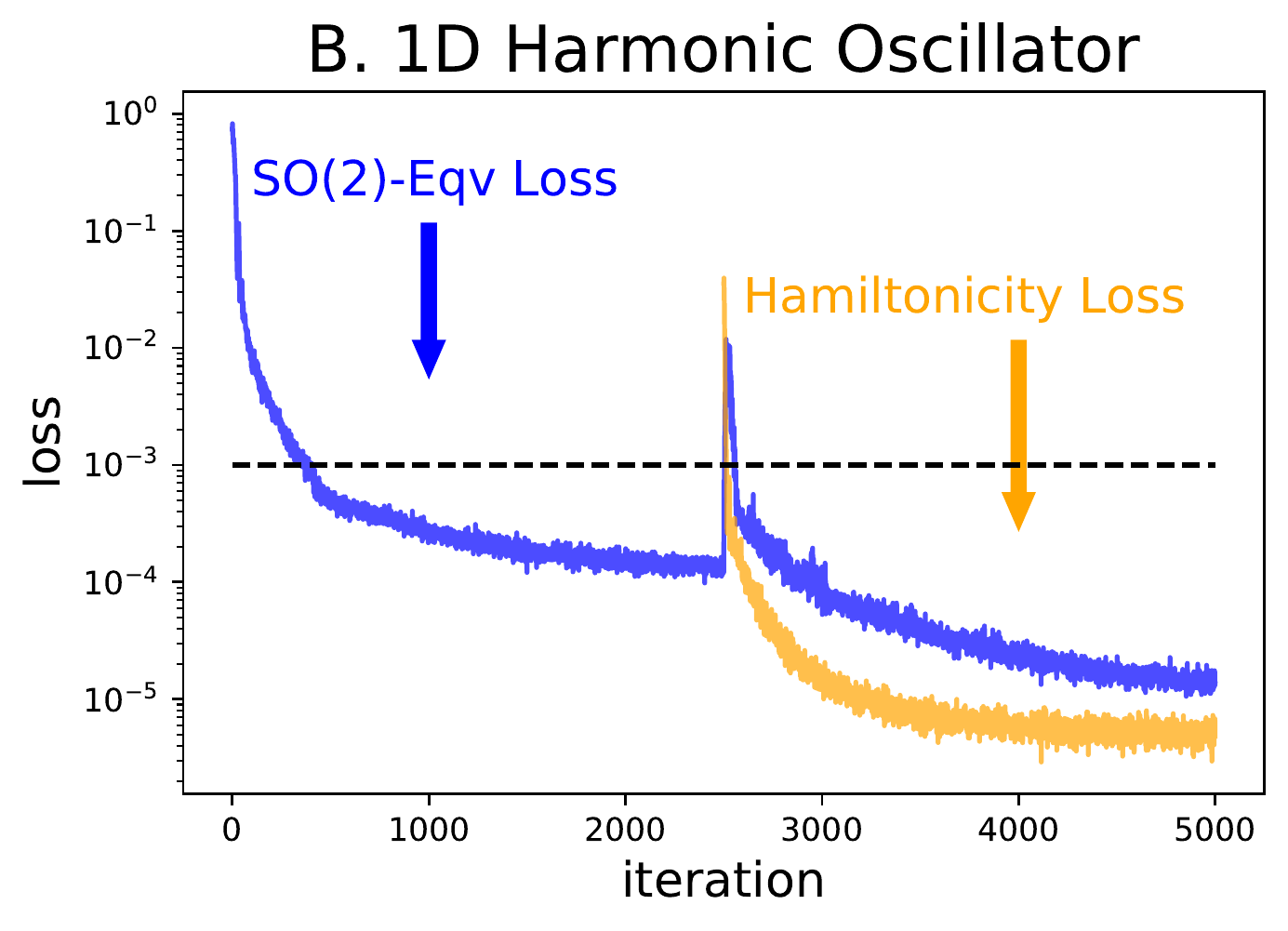} & \includegraphics[width=0.4\linewidth]{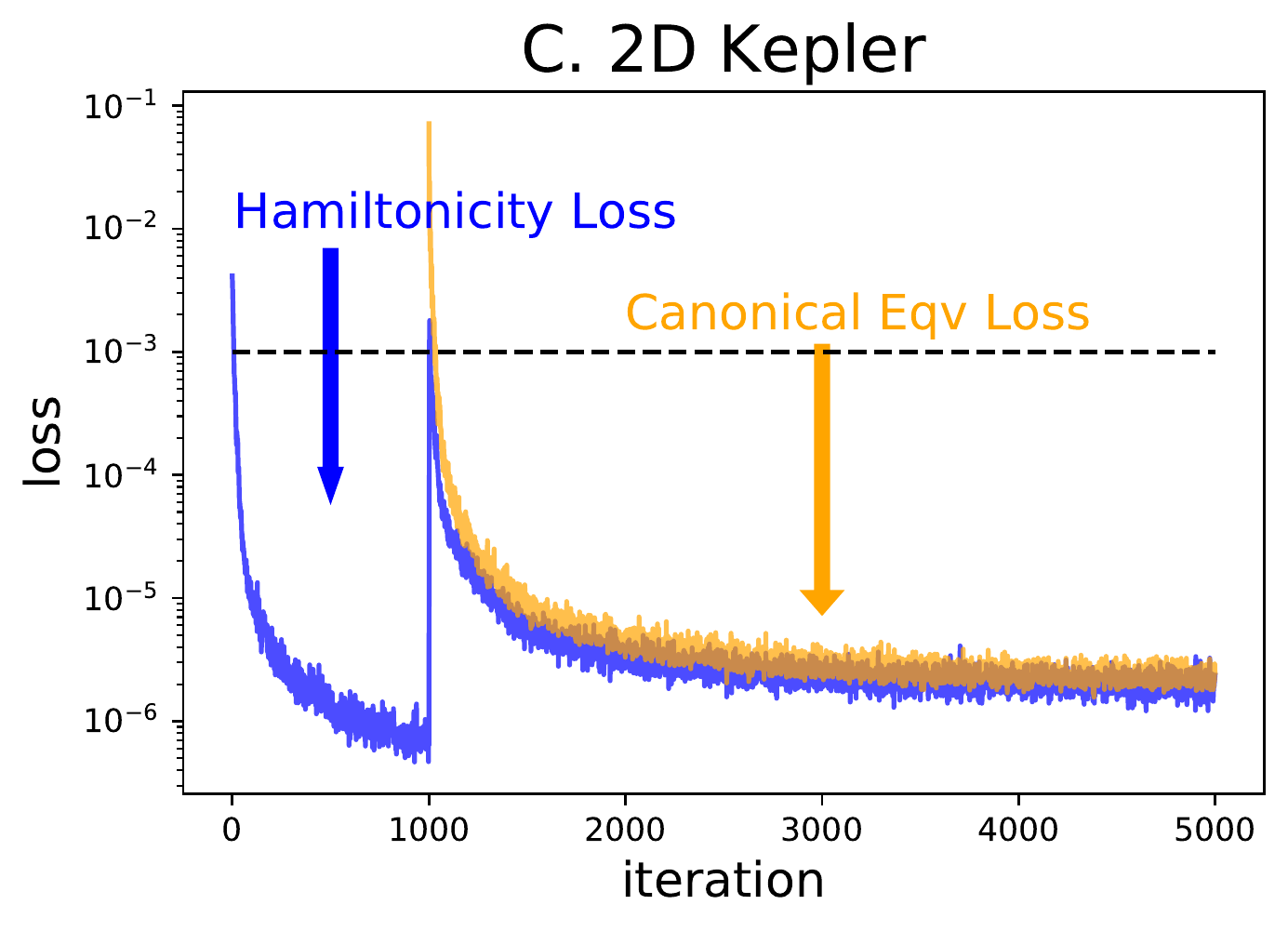} \\\hline
     \includegraphics[width=0.4\linewidth]{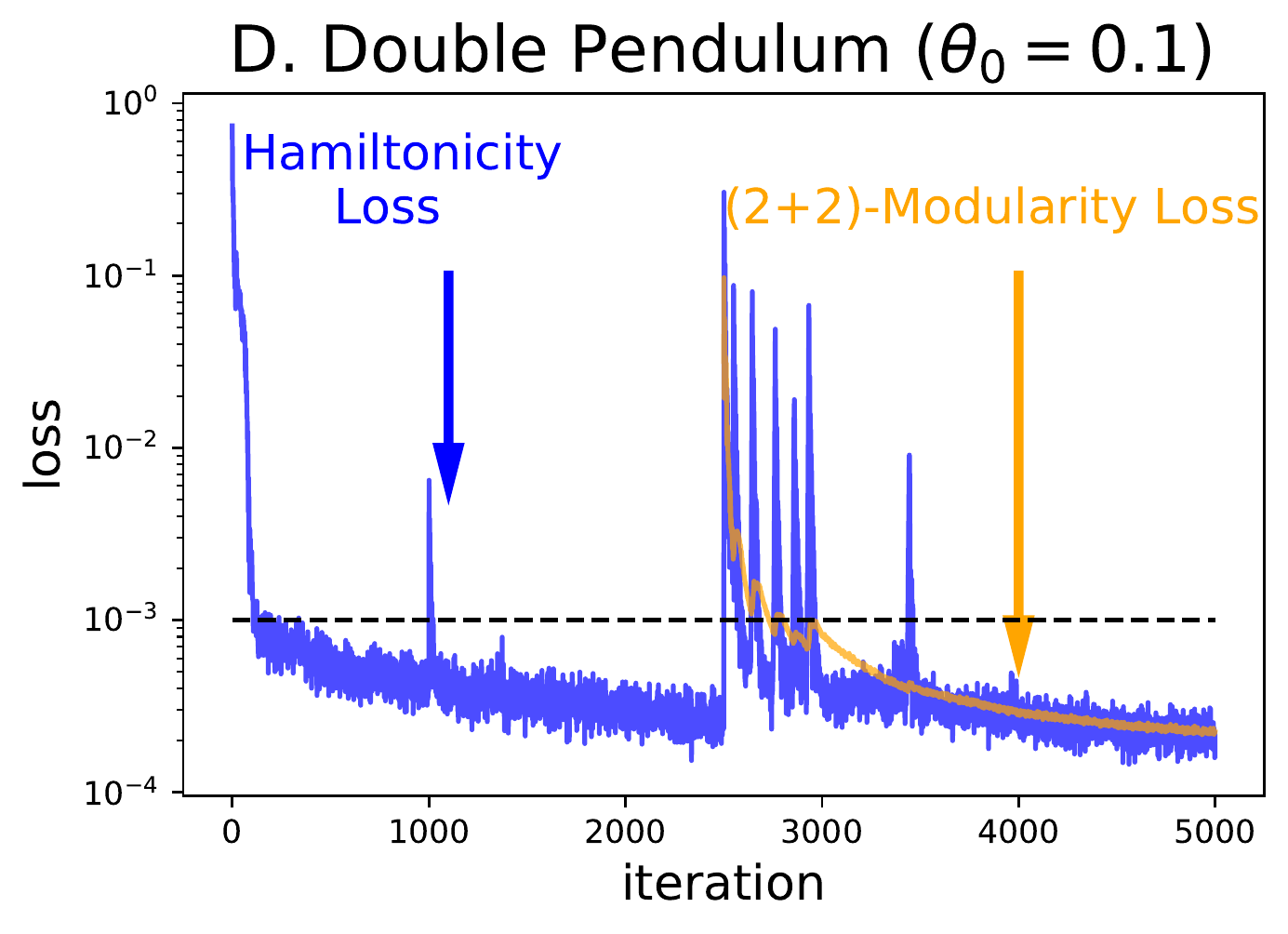} &  \includegraphics[width=0.4\linewidth]{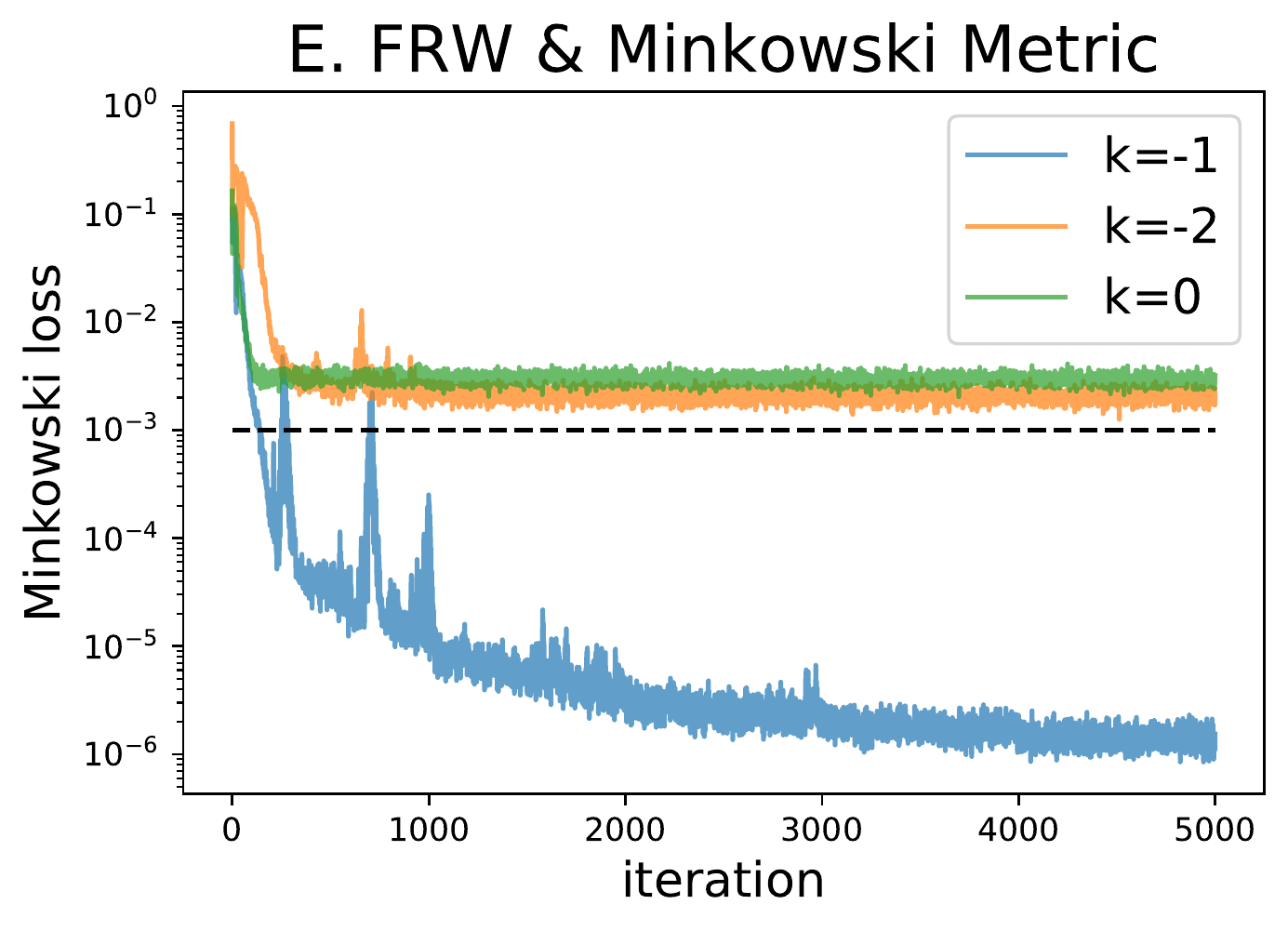} & \includegraphics[width=0.4\linewidth]{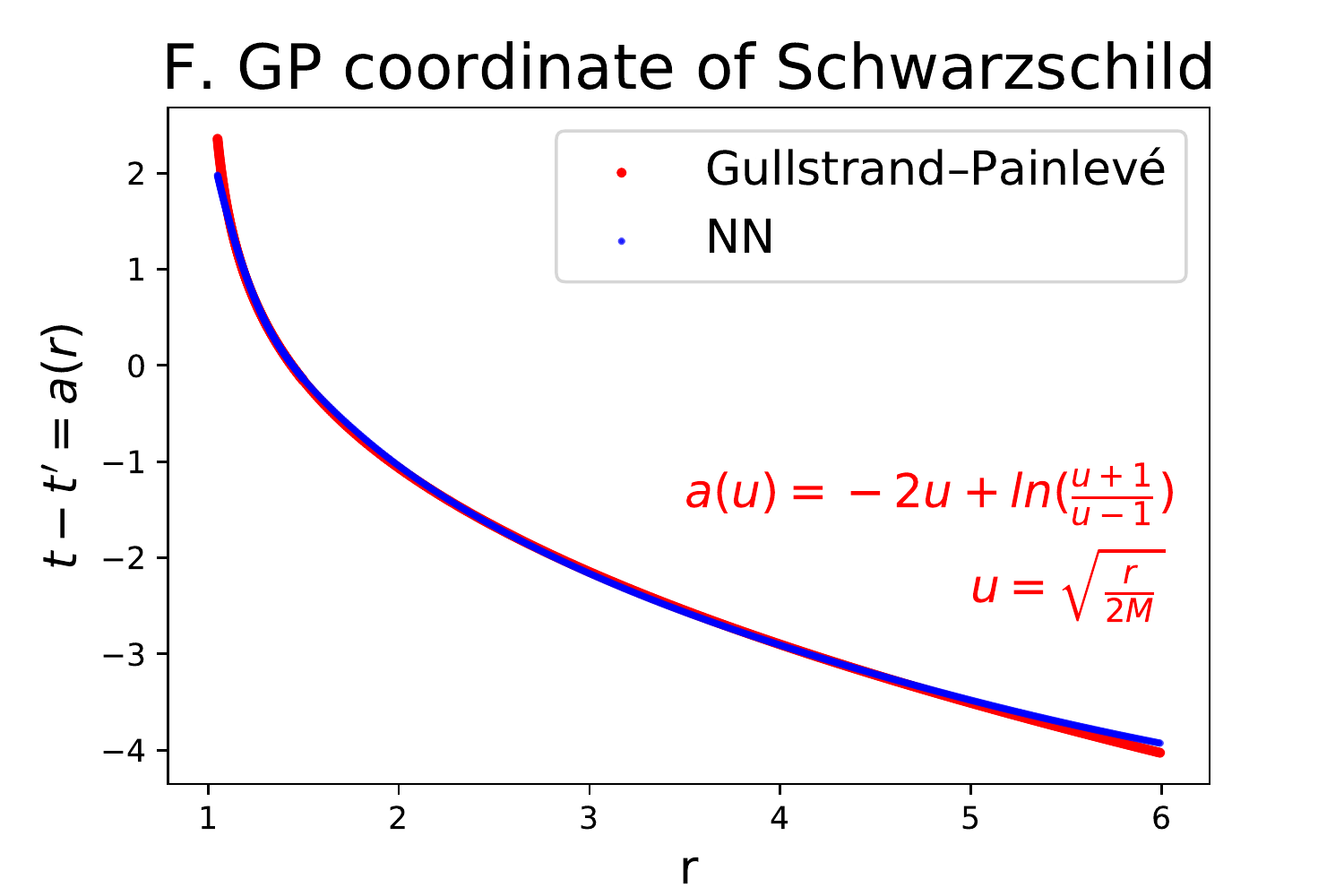} \\\hline
    \end{tabular}}
    \caption{All hidden symmetries in six tested systems are discovered by our algorithm. The last figure shows that the neural network accurately learns the Gullstrand-Painlev{\'e} transformation.}
    \label{fig:results}
\end{figure*}

We will now test our algorithm on 6 physics examples, ranging from classical mechanics to general relativity. 
\Tabl{tab:examples} summarizes these examples, labeled A, B,...,F for easy reference, listing their manifestly 
non-symmetric equations, their simplifying coordinate transformations, their transformed and manifestly symmetric equations, and their discovered hidden symmetries. As we will see, all the symmetries we had hidden in our test examples were rediscovered by our algorithm. The only example is the transformation for $A$, where the problem is so simple that an infinite family of transformations give equal symmetry.

\subsection{Warmup examples}

To build intuition for how our method works, we first apply it to the simple warmup examples A, B and C, corresponding to systems involving free motion, harmonic oscillation or Kepler problem whose simplicity has been obfuscated by various coordinate transformations.
For our examples, we consider a hidden symmetry to have been tentatively discovered if its corresponding loss drops below a threshold $\epsilon=10^{-3}$~\footnote{How low the loss should be to warrant interpretation as a symmetry discovery depends on both training accuracy and data noise; see appendix G for details.
For our examples, we found $\epsilon=10^{-3}$ to be an small enough for symbolic regression to be able to discover the exact formula, after which the symmetry loss drops to exactly zero.}. 
If that happens, we apply the AI Feynman symbolic regression package \cite{aifeynman2,aifeynman} to try to discover a symbolic approximation of the learned transformation $\fNN$ that makes the symmetry loss zero. As can be seen in \tabl{tab:examples} and \fig{fig:results}, all hidden symmetries are successfully discovered together with the coordinate transformations that reveal them. This includes not only traditional hidden symmetries such as translational invariance (example A) and rotational equivariance (example B), but also Hamiltonicity and modularity.

A-C were toy examples in the sense that we had hidden the symmetries by deliberate obfuscation.
In contrast, the value of our algorithm lies in its ability to discover symmetries hidden not by people but by nature, as in example D (the linearized double pendulum). 
We see that our method auto-discovers both hamiltonicity (by finding the correct conjugate momentum variables) and modularity (by auto-discovering the two normal modes), even though neither of these symmetries were 
manifest in the most obvious physical coordinates (the pendulum angles and angular velocities)
\cite{goldstein2002classical}.

\subsection{General relativity examples}


As a first general relativity (GR) application (example E), we consider the Friedmann-Robertson-Walker (FRW) metric\footnote{We use $r\equiv\sqrt{x^2+y^2+z^2}$ for brevity in \protect\tabl{tab:examples}, but not in the neural network, which actually takes
$(t,x,y,z)$ as inputs.} 
describing a homogeneous and isotropic expanding universe with negative spatial curvature ($k=1$) and cosmic scale factor evolution $a(t)=t$.
A GR expert will realize that its Riemann tensor vanishes everywhere, so that there must exist a coordinate transformation revealing this to be simply empty space in disguise, with Poincar\'e symmetry (Lorenz symmetry and 4D translational symmetry). Discovering this transformation is non-trivial, and is sometimes assigned as a homework problem in graduate courses. 

It is easy to show that any metric with Poincar\'e symmetry must be a multiple of the Minkowski metric $\eta$, so we define our Poincar\'e symmetry loss as $\ell\equiv {\langle ||T(\z)-\eta||^2\rangle/\langle||T(\z)||^2\rangle}$.
\Fig{fig:results} (E) shows that the Minkowski loss drops below $10^{-3}$, indicating that the $k=-1$ FRW metric is indeed homomorphic to Minkowski space, while 
the loss gets stuck above $10^{-3}$ for $k=-2$ and $k=0$.
Applying the AI Feynman symbolic regression package \cite{aifeynman2} to 
the learned transformation $\fNN=(t,x,y,z)$ reveals the exact formula
$(x',y',z',t')=(tx,ty,tz,t\sqrt{1+x^2+y^2+z^2})$, which gives vanishing loss.

We now turn to studying the spacetime of a non-rotating black hole described by the Schwarzschild metric (without loss of generality, we set $2M=1$).
This problem proved so difficult that it took physicists 17 years to clear up the misconception that something singular occurs at the event horizon, until it was finally revealed that the apparent singularity at $r=2M$ was merely caused by a poor choice of coordinates \cite{Finkelstein,Kruskal,GullstrandPainleve,GullstrandPainleve2}, just as the $z$-axis is merely a coordinate singularity in spherical coordinates.
Our method auto-discovers hidden translational symmetry in the spatial coordinates $(x,y,z)$, revealed by the coordinate transformation $t'=t+2M\left[2u+\ln{u-1\over u+1}\right]$,
where $u\equiv\sqrt{r/2M}$,
which is auto-discovered by applying AI Feynman \cite{aifeynman2} to the learned transformation $\fNN$
(see \fig{fig:results}, panel F). Since both the original and target metrics have the ${\rm SO}(3)$ (rotational) spatial symmetry, our neural network parametrizes the coordinate transformation $(x,y,z,t)\to(x',y',z',t')$ via a two-dimensional transformation $(r,t)\to (r',t')$, where $r\equiv \sqrt{x^2+y^2+z^2}$.
This transforms the Schwarzschild metric into the famous Gullstrand-Painlev\'e metric \cite{GullstrandPainleve,GullstrandPainleve2}, 
which is seen to be perfectly regular at the event horizon and can be intuitively interpreted simply as flat space flowing inward with the classical escape velocity \cite{MisnerThorneWheelerBook,Hamilton_2008}.

\section{Conclusions}          

We have presented a machine-learning algorithm for auto-discovering hidden symmetries,
and shown it to be effective for a series of examples
from classical mechanics and general relativity. 
Our symmetry definition is very broad, corresponding to the data satisfying a differential equation,
which encompasses both traditional invariance and equivariance as well as Hamiltonicity and modularity. 

Our work is linked to Noether's theorem~\cite{Noether1918}, which states that a continuous symmetry leaving the Lagrangian invariant corresponds to a conservation law. The Lagrangian is a scalar, a special case of the tensors of this paper. If we rewrite the dynamical equations in the form of  Euler-Lagrange equations, then the invariance of the Lagrangian under a symmetry group is equivalent to the equivariance of the dynamical equation under the same symmetry group, both of which imply  the same conservation laws. 

In future work, it will be interesting to seek hidden symmetries in data from both experiments and numerical simulations. Although our examples involved no more than two symmetries at once, it is 
straightforward to auto-search for a whole library of common symmetries, 
adopting the best-fitting one and recursively searching for more hidden symmetries until all are found.

Currently, our method can only search for symmetries from a list of candidates pre-specified by the user,
and cannot search for unknown symmetries. In future work, it will be interesting to enable search also for unknown symmetries, \eg, by making the Lie generators trainable. In other words, if there is {\it any} differential equation that a suitably transformed dataset satisfies, one would seek to auto-discover both the transformation and the differential equation.

{\bf Acknowledgement} We thank the Center for Brains, Minds, and Machines (CBMM)
for hospitality. This work was supported by The Casey
and Family Foundation, the Foundational Questions Institute, the Rothberg Family Fund for Cognitive Science and IAIFI through NSF grant PHY-2019786.

\bibliography{modularity}

\clearpage

\onecolumngrid

\begin{center}
    {\LARGE\bf Supplementary material}
\end{center}

\setcounter{secnumdepth}{2}

\begin{appendix}

\section{Neural network training details}

{\bf Preparing data} Our default method for generating training data is to draw the sample points of $\z$ as i.i.d.~normalized Gaussians, \ie, $\z\sim N(0,\mathbf{I}_{N})$. However, there are three issues to note: (1) For the double pendulum, since we focus on the small angle regime, we instead use the narrowed the Gaussian distribution  $(\theta_1,\theta_2,\dot{\theta}_1,\dot{\theta}_2)\sim N(0,0.1^2)$. (2) To avoid the $r=2M=1$
singularity of the Schwarzchild metric, we draw radius its $r$ from a uniform distribution $U(1.1,6)$ (F); (3) For the two general relativity examples E and F, the time variable is sampled from a uniform distribution $U[0,3]$.

The ADAM optimizer~\cite{kingma2014adam} is employed to train the neural network. The output $\z'$ is computed as $\z'=\z+\f_{\rm NN}(\z)$ rather than $\z'=\f_{\rm NN}(\z)$ to ensure $\W\equiv\frac{\partial \z'}{\partial \z}\approx \mathbf{I}$ at initialization, so that $\W^{-1}$ is well-conditioned and avoids training instabilities. 
For examples A-D, we train the neural network for 2,000 epochs with learning rate $10^{-3}$; For examples E and F, we train for 1,000 epochs three times while annealing the learning rate as $\{5\times 10^{-3},10^{-3},2\times 10^{-4}\}$.

\section{Tensor transformation rules}\label{app:tensor_transf}

Although tensor transformation rules are well-known, we list the relevant ones here for the convenience of any reader wishing to implement our method. 
We consider a coordinate vector $\z\in\mathbb{R}^N$ and a tensor field $T(\z)$. Under the coordinate transformation $\z\to\z'$, the transformation rule for the tensor field $T(\z)\to T'(\z')$ depends on the type of $T$. In general relativity, a contravariant (1,0) vector $v^i$ and a covariant (0,1) vector $v_i$ transforms as 
\begin{equation}
\begin{aligned}
    &v'^{i}=\frac{\partial z'^{i}}{\partial z^j}v^j=\W^i_{\ j}v^j,\\
    &v'_{i}=\frac{\partial z^{j}}{\partial z'^i}v_j=(\W^{-1})^j_{\ i}v_j=(\W^{-T})^{\ j}_iv_j,
\end{aligned}
\end{equation}
where $\W^i_{\ j}\equiv \frac{\partial z'^{i}}{\partial z^j}$.
More generally, an $(m,n)$ tensor $T^{i_1\cdots i_m}_{j_1\cdots j_n}$ transforms as \begin{equation}
    T'^{i'_1\cdots i'_m}_{j'_1\cdots j'_n}=\W^{i'_1}_{\ i_1}\cdots \W^{i'_m}_{\ i_m}(\W^{-1})^{j_1}_{\ j'_1}\cdots (\W^{-1})^{j_n}_{\ j'_n}T^{i_1\cdots i_m}_{j_1\cdots j_n}
\end{equation}

In this paper, we are interested in these specific cases:
\begin{itemize}
    \item $(1,0)$ vector: The transformation rule is $\f\to\f'=\W\f$, and the dynamical system $\dot{\z}=\f(\z)$ where $\dot{\z}\equiv\frac{d\z}{dt}$ falls in this category. 
    \item $(0,2)$ tensor: The transformation rule is $\g\to\g'=\W^{-T}\g \W^{-1}$, and the
    metric tensor $g_{\mu\nu}$ from General Relativity lies in this category. 
\end{itemize}

We define the first-order differentiation of $T$ wrt $\z$ as $\J\equiv\nabla T$ or $J^{i_1\cdots i_m}_{j_1\cdots j_nl}\equiv \partial_l T^{i_1\cdots i_m}_{j_1\cdots j_n}$. Written as such, $\J$ is not a $(m,n+1)$ tensor because its transformation rule is
\beq{TransformationEq1}
\begin{aligned}
J'^{i'_1\cdots i'_m}_{j'_1\cdots j'_n l'}&\equiv \partial'_{l'} T'^{i'_1\cdots i'_m}_{j'_1\cdots j'_n}=(\W^{-T})_{l'}^{\ l}\partial_l (\W^{i'_1}_{\ i_1}\cdots \W^{i'_m}_{\ i_m}(\W^{-1})^{j_1}_{\ j'_1}\cdots (\W^{-1})^{j_n}_{\ j'_n}T^{i_1\cdots i_m}_{j_1\cdots j_n})\\
&=\sum_{a=1}^m (\W^{-T})_{l'}^{\ l}(\partial_l \W^{i'_a}_{\ i_a})(\prod_{b=1,b\neq a}^m\W^{i'_b}_{\ i_b})(\prod_{c=1}^n(\W^{-1})^{j_c}_{\ j'_c})T^{i_1\cdots i_m}_{j_1\cdots j_n}\\
&+\sum_{a=1}^n (\W^{-T})_{l'}^{\ l}(\prod_{b=1}^m\W^{i'_b}_{\ i_b})(\partial_l (\W^{-1})^{j_a}_{\ j'_a} )(\prod_{c=1,c\neq a}^n(\W^{-1})^{j_c}_{\ j'_c})T^{i_1\cdots i_m}_{j_1\cdots j_n}\\
&+(\W^{-T})_{l'}^{\ l}\W^{i'_1}_{\ i_1}\cdots \W^{i'_m}_{\ i_m}(\W^{-1})^{j_1}_{\ j'_1}\cdots (\W^{-1})^{j_n}_{\ j'_n}\partial_l T^{i_1\cdots i_m}_{j_1\cdots j_n}
\end{aligned}
\eeq

$(1,0)$ vector $\f$:

\beq{TransformationEq2}
    \begin{aligned}
    \J'^{i}_{\ j}&\equiv \partial'_j \f'^i \\
    & = (\W^{-T})_{j}^{\ k}\partial_k(\W^i_{\ l}\f^l) \\ 
    & = (\W^{-T})_{\ j}^k(\partial_k \W_{\ l}^i)\f^l + (\T^{-T})_{j}^{\ k}\W^i_{\ l}\partial_k \f^l
    \end{aligned}
\eeq

$(0,2)$ tensor $\g$:

\beq{TransformationEq4}
    \begin{aligned}
    \J'_{ijk}&\equiv\partial'_k g'_{ij}\\
    &=(\W^{-T})_k^{\ l}\partial_l((\W^{-T})_i^{\ m}g_{mn}(\W^{-1})^n_{\ j})\\
    &=(\W^{-T})_k^{\ l}\partial_l(\W^{-T})_i^{\ m}g_{mn}(\W^{-1})^n_{\ j}+(\W^{-T})_k^{\ l}(\W^{-T})_i^{\ m}\partial_l g_{mn}(\W^{-1})^n_{\ j}+(\W^{-T})_k^{\ l}(\W^{-T})_i^{\ m}g_{mn}\partial_l(\W^{-1})^n_{\ j}
    \end{aligned}
\eeq

\section{Deriving PDEs forr generalized symmetry}\label{app:derive_pde}

Here we present more detailed derivations of the partial differential equations (PDEs) corresponding to Hamiltonicity, Lie equivariance, Lie invariance and canonical equivariance.
The PDEs for translational invariance and modularity are obvious given the definitions. 

\subsection{Hamiltonicity}

Hamiltonian systems (a.k.a. {\it symplectic systems}) can be written with the state $\z\in\mathbb{R}^{2d}$ as the concatenation of coordinates $\x\in\mathbb{R}^d$ and momenta $\p\in\mathbb{R}^d$. The Hamiltonian $H(\z)$ is a conserved quantity and governs the system evolution as $(\dot{\x}, \dot{\p})=(\partial_\p H, -\partial_\x H)$, or more compactly, $\dot{\z}=\M\nabla H$ where $\M$ is the (2,0) symplectic matrix $((\mat{0},\mat{I}_d),(-\mat{I}_d,\mat{0}))$. $\f\equiv\M\nabla H$ is a (1,0) vector. $\M$ satisfies $\M^{-1}=\M^T=-\M=-\M^{-T}$. We observe that 
\begin{equation}
    \J_{ij}=\partial_j f^i=\partial_j (\M\nabla H)^i=\partial_j(\M^{ik}\partial_k H)=\M^{ik}\partial_j\partial_k H
\end{equation}
Left multiplying by $(\M^{-1})_{mi}$ on both sides and utilizing $(\M^{-1})_{mi}\M^{ik}=\delta_m^{\ k}$ gives
\begin{equation}
    (\M^{-1}\J)_{mj}=\delta_m^{\ k}\partial_j\partial_k H=\partial_j\partial_m H.
\end{equation}
Since the right hand side $\partial_j\partial_m H=\partial_m\partial_j H$ is symmetric, $\M^{-1}\J$ is symmetric too, i.e.,
\begin{equation}
    (\M^{-1}\J)^T = \M^{-1}\J \Longrightarrow \J^T\M + \M\J=0
\end{equation}
Note that the converse is also true: if $\M\J$ is symmetric, then there exists a scalar field $H$ such that $\f=\M\nabla H$. We first introduce a lemma here:

\begin{lemma}
A vector field $\f$ can be written as the gradient of a scalar field $\phi$, i.e. $\f=\nabla \phi$, if and only if $\J\equiv\nabla\f$ is symmetric.
\end{lemma}

\begin{proof}
$\Longrightarrow$: $\J\equiv\nabla\f=\nabla\nabla\phi$, i.e. the Hessian (the matirx of second derivatives) of $\phi$, which is symmetric.

\noindent $\Longleftarrow$: $\J_{ij}=\J_{ji}$ means that $\partial_i \f_j=\partial_j \f_i$. Consider the loop integral
\begin{equation}
    \oint\f\cdot d\z =\frac{1}{2}\sum_{i}\sum_{j}\iint ( \partial_i f_j-\partial_j f_i)d\z_id\z_j=0
\end{equation}
The first equality is due to Green's theorem. This Vanishing loop integral is equivalent to saying $\f$ is a gradient field~\cite{gradient}.
\end{proof}
\noindent Using this lemma, the fact that $\M\J\equiv\M\nabla\f=\nabla(\M\f)$ is symmetric is equivalent to saying there exists a scalar function $H$ such that $\M\f=-\nabla H \Leftrightarrow \f=\M\nabla H$.

\subsection{Lie Equivariance}

By definition, a vector field $\f(\z)$ is equivariant under a Lie group $\mathcal{G}$ if $\f(g\z)=g\f(\z)$ for every group element $g\in\mathcal{G}$.
The element $g$ can be expressed in terms of generators $K_i$ $(i=1,\cdots,d)$ such that $g={\rm exp}(\theta_i K_i)$. 
When $\theta_i$ are small, $g\approx I+\theta_i K_i$, so the following equalities hold to first order in a Taylor expansion in $\theta$:
\begin{equation}
    \begin{aligned}
        \f(g\z)&=\f(\z+\theta_i K_i\z)=\f(\z)+\theta_i\nabla \f(\z)K_i\z \\
        g\f(\z)&=\f(\z)+\theta_iK_i\f(\z) \\
        \f(g\z)-g\f(\z)&=\theta_i(\nabla\f(\z)K_i\r-K_i\f(\z)) \\
        \frac{\partial(\f(g\z)-g\f(\z))}{\partial\theta_i}&=\nabla\f(\z)K_i\z-K_i\f(\z)=\J(\z)K_i\z-K_i\f(\z) \\
    \end{aligned}
\end{equation}
Hence the equivariance PDE and loss are
\begin{equation}
\label{eq:A5}
\begin{aligned}
    &\hl_i \f = \J K_i\z-K_i\f=\mat{0},\\ 
    &\ell_{\rm eqv}=\frac{1}{N_g}\sum_{j=1}^{N_g}\sum_{i=1}^{N_s}(||\J(\z^{(i)}) K_j\z^{(i)}-K_j\f(\z^{(i)})||_2^2)/(\sum_{i=1}^{N_s} ||\z^{(i)}||_2^2 ),
\end{aligned}
\end{equation}
where the aim of the denominator $||\f||_2^2$ is to make the loss function dimensionless, i.e., a (0,0) tensor.
Eq.~(\ref{eq:A5}) holds for all Lie groups -- one just needs to insert generators $K_i$ to obtain the corresponding loss for any Lie group. For the Lorentz group $SO(3,1)$ , for example, there are $N_g=6$ Lie generators in total, where first three correspond to boosting, and last three to spatial rotations.

\begin{equation}
\begin{aligned}
    K_1 = 
    \begin{pmatrix}
    0 & 1 & 0 & 0 \\
    1 & 0 & 0 & 0 \\
    0 & 0 & 0 & 0 \\ 
    0 & 0 & 0 & 0
    \end{pmatrix},
    &\ K_2 = 
    \begin{pmatrix}
    0 & 0 & 1 & 0 \\
    0 & 0 & 0 & 0 \\
    1 & 0 & 0 & 0 \\ 
    0 & 0 & 0 & 0
    \end{pmatrix},
    \ K_3 = 
    \begin{pmatrix}
    0 & 0 & 0 & 1 \\
    0 & 0 & 0 & 0 \\
    0 & 0 & 0 & 0 \\ 
    1 & 0 & 0 & 0
    \end{pmatrix},\\
    K_4 = 
    \begin{pmatrix}
    0 & 0 & 0 & 0 \\
    0 & 0 & 1 & 0 \\
    0 & -1 & 0 & 0 \\
    0 & 0 & 0 & 0
    \end{pmatrix},
    &\ K_5 = 
    \begin{pmatrix}
    0 & 0 & 0 & 0 \\
    0 & 0 & 0 & 1 \\
    0 & 0 & 0 & 0 \\
    0 & -1 & 0 & 0
    \end{pmatrix},
    \ K_6 = 
    \begin{pmatrix}
    0 & 0 & 0 & 0 \\
    0 & 0 & 0 & 0 \\
    0 & 0 & 0 & 1 \\
    0 & 0 & -1 & 0
    \end{pmatrix}
\end{aligned}
\end{equation}

\subsection{Lie Invariance}
Following the notation above, for any $(m,n)$-tensor field $T$:
\begin{equation}
    T(g\z)=T(\z)+\theta_i\nabla T(\z)K_i\z
\end{equation}
Lie invariance requests that
\begin{equation}
    T(g\z)=T(\z) \Longrightarrow \nabla T(\z) K_i\z=0\ \forall i
\end{equation}
Similarly we can construct the loss function as
\begin{equation}
    \ell_{\rm inv}=\frac{1}{N_g}\sum_{j=1}^{N_g}\sum_{i=1}^{N_s}(||\J(\z^{(i)}) K_j\z^{(i)}||_F^2)/(\sum_{i=1}^{N_s} ||\z^{(i)}||_2^2 ).
\end{equation}
Here $\J=\nabla T$ is reshaped as a matrix of size $N^{m+n}\times N$ where the $l$-th column is the vectorization of $\partial_l T^{i_1\cdots i_m}_{j_1\cdots j_n}$. Note that translational invariance can also be included in this framework by setting $K_j=\partial_j$.

\subsection{Canonical Equivariance}

Canonical equivariance requires that
\begin{equation}\label{eq:can_cond}
    \begin{aligned}
      &\f_{\x}(g\x,g^{-1}\p)=g^{-1}\f_{\x}(\x,\p), \\ 
      &\f_{\p}(g\x,g^{-1}\p)=g\f_{\x}(\x,\p),
    \end{aligned}
\end{equation}
for all $g\in\mathcal{G}$. For a Lie group $\mathcal{G}$, $g={\rm exp}(\theta_j K_j)$, where $K_j$ are generators. 
Note that $g^{-t}={\rm exp}(-\theta_jK_j^t)$. Differentiating with respect to $\theta_j$ in Eq.~\eqref{eq:can_cond} gives \begin{equation}
    \begin{aligned}
    &K_j\x\cdot\nabla_{\x}\f_{\x} - K_j^t\p\cdot\nabla_{\p}\f_{\x} = -K_j^t\f_{\x}, \\
    & K_j\x\cdot\nabla_{\x}\f_{\p} - K_j^t\p\cdot\nabla_{\p}\f_{\p} = K_j\f_{\p}
    \end{aligned}
\end{equation}

\section{Physical systems}\label{app:eq}

For convenience, we here review the well-known equations of the physical systems we test  our method on in the paper.

\subsection{Double pendulum}

The dynamics of a double pendulum can be described by two angles $(\theta_1,\theta_2)$ and corresponding angular velocities $(\dot{\theta}_1,\dot{\theta}_2)$:
\begin{equation}
\frac{d}{dt}
    \begin{pmatrix}
        \theta_1 \\
        \theta_2 \\
		\dot{\theta}_1\\
		\dot{\theta}_2
	\end{pmatrix}=
	\begin{pmatrix}
	\dot{\theta}_1\\
	\dot{\theta}_2\\
	\frac{m_2l_1\dot{\theta}_1^2{\rm sin}(\theta_2-\theta_1){\rm cos}(\theta_2-\theta_1)+m_2g{\rm sin}\theta_2+m_2l_2\dot{\theta}_2^2{\rm sin}(\theta_2-\theta_1)-(m_1+m_2)g{\rm sin}\theta_1}{(m_1+m_2)l_1-m_2l_1{\rm cos}^2(\theta_2-\theta_1)}\\
		\frac{-m_2l_2\dot{\theta}_2^2{\rm sin}(\theta_2-\theta_1)+(m_1+m_2)(g{\rm sin}\theta_1{\rm cos}(\theta_2-\theta_1)-l_1\dot{\theta}_2^2{\rm sin}(\theta_2-\theta_1)-g{\rm sin}\theta_2)}{(m_1+m_2)l_1-m_2l_1{\rm cos}^2(\theta_2-\theta_1)}
	\end{pmatrix}
\end{equation}

The canonical coordinates corresponding to $(\dot{\theta}_1,\dot{\theta}_2)$ are $(p_1,p_2)$, given by
\begin{equation}
\begin{aligned}
    &p_1 = (m_1+m_2)l_1^2\dot{\theta}_1 + m_2l_1l_2\dot{\theta}_2{\rm cos}(\theta_1-\theta_2), \\
    &p_2 = m_2l_2^2\dot{\theta}_2+m_2l_1l_2\dot{\theta}_1{\rm cos}(\theta_1-\theta_2).
\end{aligned}
\end{equation}
The dynamical equations (B1) can be rewritten in terms of $(\theta_1, \theta_2, p_1, p_2)$:
\begin{equation}
\frac{d}{dt}
    \begin{pmatrix}
    \theta_1 \\ 
    \theta_2 \\ 
    p_1 \\ 
    p_2
    \end{pmatrix}=
    \begin{pmatrix}
    \frac{l_2p_1-l_1p_2{\rm cos}(\theta_1-\theta_2)}{l_1^2l_2(m_1+m_2{\rm sin}^2(\theta_1-\theta_2))}\\
    \frac{l_1(m_1+m_2)p_2-l_1m_2p_1{\rm cos}(\theta_1-\theta_2)}{l_1l_2^2m_2(m_1+m_2{\rm sin}^2(\theta_1-\theta_2))}\\
    -(m_1+m_2)gl_1{\rm sin}\theta_1 - C_1 + C_2 \\
    - m_2gl_2{\rm sin}\theta_2 + C_1 - C_2
    \end{pmatrix},\ 
    {\rm where}\ \begin{pmatrix}
    C_1 \\ 
    C_2
    \end{pmatrix}=
    \begin{pmatrix}
    \frac{p_1p_2{\rm sin}(\theta_1-\theta_2)}{l_1l_2(m_1+m_2{\rm sin}^2(\theta_1-\theta_2))}\\
    \frac{l_2^2m_2p_1^2+l_1^2(m_1+m_2)p_2^2-l_1l_2m_2p_1p_2{\rm cos}(\theta_1-\theta_2)}{2l_1^2l_2^2(m_1+m_2{\rm sin}^2(\theta_1-\theta_2))^2}
    \end{pmatrix}
\end{equation}
In the small angle regime $\theta_1,\theta_2\ll1$, after using the approximations $\sin\theta_1\approx\theta_1$, $\sin\theta_2\approx\theta_2$, $\cos(\theta_1-\theta_2)\approx 1$, $\dot{\theta}_1^2\approx 0$, $\dot{\theta}_2^2\approx 0$ and setting $l_1=l_2=l$,  (E1)-(E3) simplify to
\begin{equation}
    \frac{d}{dt}
    \begin{pmatrix}
    \theta_1 \\ 
    \theta_2 \\
    \dot{\theta}_1 \\ 
    \dot{\theta}_2
    \end{pmatrix}
    =
    \begin{pmatrix}
    0 & 0 & 1 & 0 \\
    0 & 0 & 0 & 1 \\
    -\frac{(m_1+m_2)g}{m_1l} & \frac{m_2g}{m_1l} & 0 & 0 \\
    \frac{(m_1+m_2)g}{m_1l} & -\frac{(m_1+m_2)g}{m_1l} & 0 & 0 \\
    \end{pmatrix}
    \begin{pmatrix}
     \theta_1 \\ 
    \theta_2 \\
    \dot{\theta}_1 \\ 
    \dot{\theta}_2
    \end{pmatrix},
\end{equation}
\begin{equation}
    \begin{aligned}
        &p_1 = (m_1+m_2)l^2\dot{\theta}_1 + m_2l^2\dot{\theta}_2,\\
        &p_2 = m_2l^2\dot{\theta}_2 + m_2l^2\dot{\theta}_1,
    \end{aligned}
\end{equation}
\begin{equation}
    \frac{d}{dt}
    \begin{pmatrix}
    \theta_1 \\ 
    \theta_2 \\
    p_1 \\ 
    p_2
    \end{pmatrix}=
    \begin{pmatrix}
    0 & 0 & \frac{1}{m_1l^2} & -\frac{1}{m_1l^2} \\
    0 & 0 & -\frac{1}{m_1l^2} & \frac{m_1+m_2}{l^2m_2m_1} \\
    -(m_1+m_2)gl & 0 & 0 & 0 \\
    0 & -m_2gl & 0 & 0 
    \end{pmatrix}
    \begin{pmatrix}
    \theta_1 \\ 
    \theta_2 \\
    p_1 \\ 
    p_2
    \end{pmatrix}, \\
\end{equation}
\begin{equation}
    \begin{pmatrix}
    \theta_1 \\
    \dot{\theta}_1 \\
    \theta_2 \\
    \dot{\theta}_2 \\
    \end{pmatrix} = 
    \begin{pmatrix}
    1 & 0 & 0 & 0 \\
    0 & \frac{1}{m_1l^2} & 0 & -\frac{1}{m_1l^2} \\
    0 & 0 & 1 & 0 \\
    0 & -\frac{1}{m_1l^2} & 0 & \frac{m_1+m_2}{m_1m_2l^2}
    \end{pmatrix}
    \begin{pmatrix}
    x_1 \\
    p_1 \\
    x_2 \\
    p_2
    \end{pmatrix},
\end{equation}

From Eq.~(E4), we can write $\theta_1$ and $\theta_2$ as
linear combinatitons of two {\bf normal modes} 
$\theta_\pm(t)$ that oscillate independently with frequencies $\omega_{\pm}$:
\begin{equation}
    \begin{aligned}
       \omega_{\pm}^2=\frac{g}{m_1l}\left[m_1+m_2\pm\sqrt{(m_1+m_2)m_2}\right], \\
       \theta_1 = \theta_+ - \theta_-, \quad\theta_2 = -\sqrt{\frac{m_1+m_2}{m_2}}(\theta_++\theta_-)
    \end{aligned}
\end{equation}

\subsection{Friedmann–Robertson–Walker (FRW) metric}

The Friedmann–Robertson–Walker (FRW) metric in the spherical coordinate and in the cartesian coordiante are:

\begin{equation}
\begin{aligned}
    &\g_{\rm spherical}^{\rm FRW}(t,r,\theta,\phi)=\begin{pmatrix}
   	1 & 0 & 0 & 0 \\
   	0 & -\frac{a(t)^2}{1-kr^2} & 0 & 0 \\
   	0 & 0 & -a(t)^2r^2 & 0 \\
   	0 & 0 & 0 & -a(t)^2r^2{\rm sin}^2\theta
   	\end{pmatrix},\\
    &\g_{\rm cartesian}^{\rm FRW}(t,x,y,z)=\begin{pmatrix}
   	1 & 0 & 0 & 0 \\
   	0 & -(r^2+\frac{kx^2}{1-kr^2})\frac{a(t)^2}{r^2} & -\frac{kxya(t)^2}{1-kr^2} & -\frac{kxza(t)^2}{1-kr^2} \\
   	0 & -\frac{kxya(t)^2}{1-kr^2} & -(r^2+\frac{ky^2}{1-kr^2})\frac{a(t)^2}{r^2} & -\frac{kyza(t)^2}{1-kr^2} \\
   	0 & -\frac{kxza(t)^2}{1-kr^2} & -\frac{kyza(t)^2}{1-kr^2} & -(r^2+\frac{kz^2}{1-kr^2})\frac{a(t)^2}{r^2}
   	\end{pmatrix}.
\end{aligned}
\end{equation}
When $k=1$ and $a(t)=t$, the FRW metric can be transformed to the Minkowski metric via a global transformation:

\begin{equation}
    \begin{pmatrix} t' \\ x' \\ y' \\ z'
   	\end{pmatrix}
   	=
   	\begin{pmatrix}
   	t\sqrt{1+r^2} \\ tx \\ ty \\ tz
   	\end{pmatrix}
\end{equation}

\subsection{Schwarzschild metric and GP coordinate}

The Schwarzschild metric in spherical coordinates and in Cartesian coordinates are
\begin{equation}
\begin{aligned}
    &\g_{\rm spherical}^{\rm Sch}(t,r,\theta,\phi)=\begin{pmatrix}
   	1-\frac{2M}{r} & 0 & 0 & 0 \\
   	0 & -(1-\frac{2M}{r})^{-1} & 0 & 0 \\
   	0 & 0 & -r^2 & 0 \\
   	0 & 0 & 0 & -r^2{\rm sin}^2\theta
   	\end{pmatrix},\\
    &\g_{\rm cartesian}^{\rm Sch}(t,x,y,z)=\begin{pmatrix}
    1-\frac{2M}{r} & 0 & 0 & 0 \\
    0 & -1-\frac{2Mx^2}{(r-2M)r^2} & -\frac{2Mxy}{(r-2M)r^2} & -\frac{2Mxz}{(r-2M)r^2} \\
    0 & -\frac{2Mxy}{(r-2M)r^2} & -1-\frac{2My^2}{(r-2M)r^2} & -\frac{2Myz}{(r-2M)r^2} \\
    0 & -\frac{2Mxz}{(r-2M)r^2} & -\frac{2Myz}{(r-2M)r^2} & -1-\frac{2Mz^2}{(r-2M)r^2}
    \end{pmatrix}.
\end{aligned}
\end{equation}
The spatial part of the metric diverges at $r=2M$. However, if we transform to the
Gullstrand–Painlev\'e coordinate defined by
\begin{equation}
    \begin{pmatrix}
    		t' \\ x' \\ y' \\ z'
    	\end{pmatrix} = \begin{pmatrix}
    	t-2M(-2u+\ln{u+1\atop u-1}))\\
    	x \\ 
    	y \\ 
    	z
    	\end{pmatrix}, u=\sqrt{\frac{r}{2M}},
\end{equation}
the apparent singularity disappears, and we obtain a spatial metric which, remarkably, is 
the same as for Euclidean space:
\begin{equation}
    \g_{\rm GP}^{\rm Sch}=\begin{pmatrix}
    	1-\frac{2M}{r} & -\sqrt{\frac{2M}{r}}\frac{x}{r} & -\sqrt{\frac{2M}{r}}\frac{y}{r} & -\sqrt{\frac{2M}{r}}\frac{z}{r} \\
    	-\sqrt{\frac{2M}{r}}\frac{x}{r} & -1 & 0 & 0 \\
    	-\sqrt{\frac{2M}{r}}\frac{y}{r} & 0 & -1 & 0 \\
    	-\sqrt{\frac{2M}{r}}\frac{z}{r} & 0 & 0 & -1
    	\end{pmatrix}
\end{equation}

\section{Underconstrained problems}

We saw that, ironically, the only example where our symbolic regression failed to find 
an exact formula for the discovered coordinate transformation was the very simplest one: the uniform motion of example A: 
\begin{equation}
    \begin{aligned}
        \dot{x} = p \\
        \dot{p} = 0
    \end{aligned}
\end{equation}
The reason for this is that the final equations with manifest symmetry are so simple 
that there are infinitely many different transformations that produce it, and our neural 
network has no incentive to find the simplest one.
Specifically, any coordinate transformation of the form
$x'=c_1(p)x+c_2(p)$ and $p'=c_3(p)$ 
preserves both translational invariance 
because 
\begin{equation}
    \begin{aligned}
    &\dot{x}'=c_1(p)\dot{x}+c_1'(p)x\dot{p} + c_2'(p)\dot{p} = c_1'(p)p=c_1'(c_3^{-1}(p'))c_3^{-1}(p')\equiv A(p'), \\
    &\dot{p}'=c_3'(p)\dot{p} = 0
    \end{aligned}
\end{equation}
and Hamiltonicity $H=\int A(p')dp'$.

\section{The AI Feynman Package}\label{app:feynman}

The goal of {\it symbolic regression} is discovering a symbolic expression that accurately matches a given data set. More specifically, we are given a table of numbers whose rows are of the form $\{x_1,\cdots,x_n,y\}$ where $y=f(x_1,\cdots,x_n)$, and our task is to discover the correct symbolic expression for the unknown mystery function $f$. The symbolic regression problem for mathematical functions has been tackled with a variety of methods, including sparse regression and genetic algorithms. 

The {\it AI Feynman} software improves on these methods by using physics-inspired strategies enabled by neural networks. 
It uses neural networks to discover hidden simplicity such as symmetry and separability in the data, which enables it to recursively break harder problems into simpler ones with fewer variables. The overall algorithm consists of a series of modules that try to exploit each of the above-mentioned properties. Like a human scientist, it tries many different strategies (modules) in turn, and if it cannot solve the full problem in one fell swoop, it tries to transform it and divide it into simpler pieces that can be tackled separately, recursively re-launching the full algorithm on each piece. These modules include dimensional analysis, polynomial fitting, brute force experssion search, neural-network-based tests and transformations (symmetries and separability).

Note that AI Feynman can handle cases with multiple input variables, but only one output, so we run AI Feynman multiple times to obtain each component of the transformation/transformed equation. 
We will now illustrate how AI Feynman works on Example E and F.

For Example E (Expanding universe and empty space), this is the transformation to be discovered:
\begin{equation*}
    \begin{pmatrix} t' \\ x' \\ y' \\ z'
   	   \end{pmatrix}=\begin{pmatrix}
   	   t\sqrt{1+r^2} \\ tx \\ ty \\ tz
   	   \end{pmatrix}
\end{equation*}
We thus need to run AI Feynman four times to learn $(t',x',y',z')$ separately. (1) for $t'$ (the most complicated case), by training a neural network, AI Feynman is first able to first discover $SO(3)$ symmetry, i.e., that $t'$ depends on $x$, $y$ and $z$ only through their combination $r\equiv \sqrt{x^2+y^2+z^2}$. AI Feynman then tries to express $t'$ as a function of $t$ and $r$. When AI Feynman tries to square the output $t'^2$ and do polynomial fitting, it succeeds and find that $t'^2=t^2+t^2r^2$. (2) for $x'$, AI Feynman can discover $x'=tx$ immediately by trying polynomial fitting, and the same holds for $y'$ and $z'$.

For Example F (Schwarzchild black hole and GP coordinates), this is the transformation to be discovered:
\begin{equation*}
    \begin{pmatrix}
    		t' \\ x' \\ y' \\ z'
    	\end{pmatrix} = \begin{pmatrix}
    	t+2M\left[2\sqrt{\frac{r}{2M}}+\ln{\sqrt{\frac{r}{2M}}-1\over \sqrt{\frac{r}{2M}}+1}\right]\\
    	x\\ 
    	y\\ 
    	z
    	\end{pmatrix}
\end{equation*}
In our case, $2M=1$. We again need to run AI Feynman four times to learn $(t',x',y',z')$ separately. (1) Like in the previous example, AI Feynman first discover rotational symmetry, i.e., that $t'$ only depends on the apace coordinates via their combination $r\equiv\sqrt{x^2+y^2+z^2}$. AI Feynman then discover additive modularity, i.e.,  that $t'(t,r)=g(t)+h(r)$ for some yet-to-be-discovered functions $g$ and $h$, so the problem break down into two simpler problems. $g(t)=t$ is trivially discovered via polynomial fitting, while $h(r)=\left[2\sqrt{r}+\ln{\sqrt{r}-1\over\sqrt{r}+1}\right]$ 
is discovered by brute force search over ever-more-complex symbolic formulas. (2) $x'=x$ can be easily discovered by polynomial fitting, as well as $y'=y$ and $z'=z$.

\section{Noisy data}
In the main text, we presented results for noise-free data. In an experimental setting, we need to quantify how noise affects affect our results. 
We construct noisy data by adding Gaussian noise of standard deviation $\sigma$ to each component of $T(\z)$ and rerun Example E (the Expanding universe and Minkowski space) for five different noise levels. 
FIG.~\ref{fig:frw_noise} shows that the effect of data noise is simply to add a noise floor to the resulting loss; in this particular example, we see that data noise with $\sigma=0.01$ still allows 
the loss to drop significantly below 0.001 and the models without hidden symmetry in Figure 3E.

\begin{figure}
    \centering
    \includegraphics[width=0.6\linewidth]{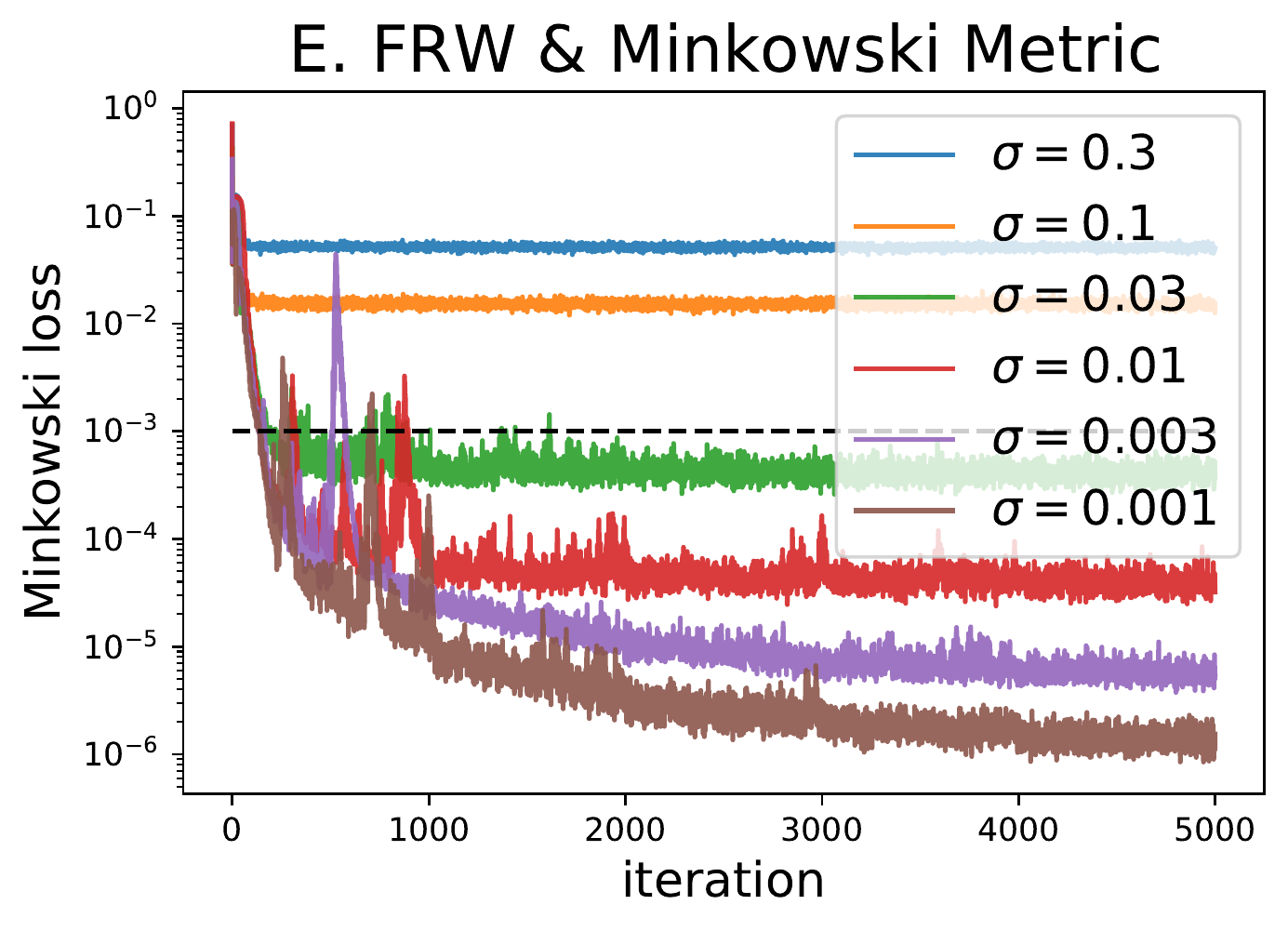}
    \caption{Training curve for the noisy metric tensor in Example E.}
    \label{fig:frw_noise}
\end{figure}

\section{Full Table \ref{tab:examples}}

A more detailed version of Table \ref{tab:examples} is Table \ref{tab:examples_full}, including $(m,n,s,t)$ indices and results of symbolic discoveries.

\begin{table*}[]
    \centering
    \caption{Physical Systems studied (Full version of Table \ref{tab:examples})}
    \resizebox{\textwidth}{!}{%
    \begin{tabular}{|c|c|c|c|c|c|c|c|}\hline
    ID & Name  & Original dynamics $\dot\z=\f(\z)$ or metric $g(\z)$ & Transformation $\z\mapsto\z'$& Symmetric dynamics 
    $\dot\z'=\f'(\z')$ or metric $g'(\z')$ &  Manifest Symmetries & $(m,n,s,t)$ &Symbolic sol.?\\\hline
    A & \makecell{1D Uniform \\ Motion }  & $\frac{d}{dt}\begin{pmatrix}
	a \\ b
	\end{pmatrix}=\begin{pmatrix}
	\frac{1}{2}(a+b){\rm ln}(\frac{a-b}{2})\\
	\frac{1}{2}(a+b){\rm ln}(\frac{a-b}{2})
	\end{pmatrix}$ & $\begin{pmatrix}
	a \\ b
	\end{pmatrix}=\begin{pmatrix}
	e^{\frac{1}{2}x} + e^{\frac{1}{2}p}\\ e^{\frac{1}{2}x} - e^{\frac{1}{2}p}
	\end{pmatrix}$ & $\frac{d}{dt}\begin{pmatrix}
    x\\ p
    \end{pmatrix}=\begin{pmatrix}
    p\\0
    \end{pmatrix}$ & \makecell{Hamiltonicity \\ 1D Translational invariance} & \makecell{$(1,0,-1,0)$ \\ $(1,0,-1,0)$} & No, Yes \\\hline
    B & \makecell{1D Harmonic \\Oscillator} & $\frac{d}{dt}\begin{pmatrix}
    a \\ b 
    \end{pmatrix}=\begin{pmatrix}
    (1+a){\rm ln}(1+b) \\ 
    -(1+b){\rm ln}(1+a)
    \end{pmatrix}$ & $\frac{d}{dt}\begin{pmatrix}
    a \\ b
    \end{pmatrix}=\begin{pmatrix}
    e^{\frac{1}{2}x}-1 \\ e^{\frac{1}{2}p}-1
    \end{pmatrix}$ &$\frac{d}{dt}\begin{pmatrix}
    x \\ p
    \end{pmatrix}=\begin{pmatrix}
    p \\ -x
    \end{pmatrix}$ & \makecell{Hamiltonicity \\ SO(2)-equivariance} & \makecell{$(1,0,-1,0)$ \\ $(1,0,0,1)$} & Yes, Yes  \\\hline
    C & 2D Kepler &  $\frac{d}{dt}\begin{pmatrix}
    a \\ b \\ c \\ d
    \end{pmatrix}=\begin{pmatrix}
    (1+a){\rm ln}(1+b) \\
    -\frac{(1+b){\rm ln}(1+a)}{8({\rm ln}^2(1+a)+{\rm ln}^2(1+c))^{3/2}}\\
    (1+c){\rm ln}(1+d) \\
    -\frac{(1+d){\rm ln}(1+c)}{8({\rm ln}^2(1+a)+{\rm ln}^2(1+c))^{3/2}}\\
    \end{pmatrix}$ &  $\begin{pmatrix}
    a \\ b \\ c \\ d
    \end{pmatrix}=\begin{pmatrix}
    e^{\frac{1}{2}x}-1 \\ e^{\frac{1}{2}p_x}-1 \\ e^{\frac{1}{2}y}-1 \\ e^{\frac{1}{2}p_y}-1
    \end{pmatrix}$  & $\frac{d}{dt}\begin{pmatrix}
    x \\ p_x \\ y \\ p_y
    \end{pmatrix}=\begin{pmatrix}
    p_x \\ -x/r^3 \\ p_y \\ -y/r^3
    \end{pmatrix}$ & \makecell{Hamiltonicity \\Can SO(2)-equivariance} & \makecell{$(1,0,-1,0)$\\$(1,0,0,1)$} & Yes, Yes \\\hline   
    D & \makecell{Double \\ Pendulum} & $\frac{d}{dt}\begin{pmatrix}
    \theta_1 \\ \dot{\theta}_1 \\ \theta_2 \\ \dot{\theta}_2
    \end{pmatrix}=\begin{pmatrix}
    \dot{\theta}_1 \\ -\frac{(m_1+m_2)g}{m_1l}\theta_1+\frac{m_2g}{m_1l}\theta_2 \\ \dot{\theta}_2 \\ \frac{(m_1+m_2)g}{m_1l}\theta_1-\frac{(m_1+m_2)g}{m_1l}\theta_2 
    \end{pmatrix}$ & \makecell{
    $\begin{pmatrix}
    \theta_1 \\ 
    \theta_2 \\ 
    \dot{\theta}_1 \\ \dot{\theta}_2
    \end{pmatrix}=\begin{pmatrix}
    -1 & 1 & 0 & 0 \\
     a & a & 0 & 0 \\
    0 & 0 & -1 & 1 \\
    0 & 0 & a & a
    \end{pmatrix}\begin{pmatrix}
    \theta_+ \\ \theta_- \\ \dot{\theta}_+ \\ \dot{\theta}_-
    \end{pmatrix}$ \\
    $a=\sqrt{\frac{m_1+m_2}{m_2}}$} &
    \makecell{$\frac{d}{dt}\begin{pmatrix}
    \theta_+ \\ \dot{\theta}_+ \\ \theta_- \\ \dot{\theta}_-
    \end{pmatrix}=\begin{pmatrix}
    \dot{\theta}_+ \\ -\omega_+^2\theta_+ \\ \dot{\theta}_- \\ -\omega_-^2\theta_-
    \end{pmatrix}$ \\  $\omega_{\pm}^2=\frac{(m_1+m_2)g}{m_1l}(1\pm\sqrt{\frac{m_2}{m_1+m_2}})$}  & \makecell{Hamiltonicity \\ $(2+2)$-Modularity} & \makecell{$(1,0,-1,0)$ \\ $(1,0,-1,0)$} & Yes, Yes  \\\hline
    E & \makecell{Expanding\\universe\\\& empty space} & $\g=\begin{pmatrix}
   	1 & 0 & 0 & 0 \\
   	0 & -(r^2+\frac{kx^2}{1-kr^2})\frac{t^2}{r^2} & -\frac{kxyt^2}{1-kr^2} & -\frac{kxzt^2}{1-kr^2} \\
   	0 & -\frac{kxyt^2}{1-kr^2} & -(r^2+\frac{ky^2}{1-kr^2})\frac{t^2}{r^2} & -\frac{kyzt^2}{1-kr^2} \\
   	0 & -\frac{kxzt^2}{1-kr^2} & -\frac{kyzt^2}{1-kr^2} & -(r^2+\frac{kz^2}{1-kr^2})\frac{t^2}{r^2}
   	\end{pmatrix}$ & $\begin{pmatrix} t' \\ x' \\ y' \\ z'
   	   \end{pmatrix}=\begin{pmatrix}
   	   t\sqrt{1+r^2} \\ tx \\ ty \\ tz
   	   \end{pmatrix}$ & $\g=\begin{pmatrix}
   	1 & 0 & 0 & 0 \\
   	0 & -1 & 0 & 0 \\
   	0 & 0 & -1 & 0 \\
   	0 & 0 & 0 & -1
   	\end{pmatrix}$ & \makecell{SO(3,1)-Invariance \\ 4D Translational Invariance} & \makecell{$(0,2,0,-2)$\\$(0,2,-1,-3)$} & Yes, Yes \\\hline
    F & \makecell{Schwarzchild\\black hole\\\& GP metric} & $\g=\begin{pmatrix}
    1-\frac{2M}{r} & 0 & 0 & 0 \\
    0 & -1-\frac{2Mx^2}{(r-2M)r^2} & -\frac{2Mxy}{(r-2M)r^2} & -\frac{2Mxz}{(r-2M)r^2} \\
    0 & -\frac{2Mxy}{(r-2M)r^2} & -1-\frac{2My^2}{(r-2M)r^2} & -\frac{2Myz}{(r-2M)r^2} \\
    0 & -\frac{2Mxz}{(r-2M)r^2} & -\frac{2Myz}{(r-2M)r^2} & -1-\frac{2Mz^2}{(r-2M)r^2}
    \end{pmatrix}$ & \makecell{
    	$\begin{pmatrix}
    		t' \\ x' \\ y' \\ z'
    	\end{pmatrix} = \begin{pmatrix}
    	t+2M\left[2u+\ln{u-1\over u+1}\right]\\
    	x\\ 
    	y\\ 
    	z
    	\end{pmatrix}$ \\ $u\equiv\sqrt{\frac{r}{2M}}$} &   $\g=\begin{pmatrix}
    	1-\frac{2M}{r} & -\sqrt{\frac{2M}{r}}\frac{x}{r} & -\sqrt{\frac{2M}{r}}\frac{y}{r} & -\sqrt{\frac{2M}{r}}\frac{z}{r} \\
    	-\sqrt{\frac{2M}{r}}\frac{x}{r} & -1 & 0 & 0 \\
    	-\sqrt{\frac{2M}{r}}\frac{y}{r} & 0 & -1 & 0 \\
    	-\sqrt{\frac{2M}{r}}\frac{z}{r} & 0 & 0 & -1
    	\end{pmatrix}$ & \makecell{SO(3)-Invariance \\ 3D Translational Invariance} & \makecell{$(0,2,0,-2)$ \\
    	$(0,2,-1,-3)$} & Yes, Yes \\\hline
    \end{tabular}}\\
    \label{tab:examples_full}
\end{table*}

\end{appendix}
\end{document}